\documentclass[12pt]{iopart}

\expandafter\let\csname equation*\endcsname\relax
\expandafter\let\csname endequation*\endcsname\relax

\usepackage{amsmath,amsfonts,amssymb}
\usepackage{graphicx,wrapfig}
\usepackage{floatrow}
\usepackage{amsthm}
\usepackage{orcidlink}

\usepackage{pgfplotstable}
\usepackage{pgfplots}
\pgfplotsset{compat=1.18}

\usepackage{style2}
\newtheorem{assumption}{Assumption}

\begin{document}

\title[vpal eeg]{Fast $\ell_1$-Regularized EEG Source Localization \\ Using Variable Projection}

\author{Jack Michael Solomon$^{1}$\orcidlink{0009-0001-0822-9438}, Rosemary Anne Renaut$^{2}$ \orcidlink{0000-0001-9296-0890}, Matthias Chung$^{1}$\orcidlink{0000-0001-7822-4539}}

\address{$^1$Emory University, Department of Mathematics, 400 Dowman Drive, Atlanta, GA, USA}

\address{$^2$Arizona State University, Department of Mathematics and Statistics, Wexler Hall, 901 Palm Walk Room 216, Tempe, AZ, USA}

\eads{\mailto{jack.michael.solomon@emory.edu}, \mailto{renaut@asu.edu}, \mailto{matthias.chung@emory.edu}}

\begin{abstract}
    Electroencephalograms (EEG) are invaluable for treating neurological disorders, however, mapping EEG electrode readings to brain activity requires solving a challenging inverse problem.  Due to the time series data, the use of $\ell_1$ regularization quickly becomes intractable for many solvers, and, despite the reconstruction advantages of $\ell_1$ regularization, $\ell_2$-based approaches such as {\tt sLORETA} are used in practice. In this work, we formulate EEG source localization as a graphical generalized elastic net inverse problem and present a \emph{variable projected} algorithm ({\tt VPAL}) suitable for fast EEG source localization. We prove convergence of this solver for a broad class of separable convex, potentially non-smooth functions subject to linear constraints and include a modification of {\tt VPAL} that reconstructs time points in sequence, suitable for real-time reconstruction. Our proposed methods are compared to state-of-the-art approaches including {\tt sLORETA} and other methods for $\ell_1$-regularized inverse problems.

\end{abstract}


\section{Introduction}

An electroencephalogram (EEG) is a well-established diagnostic tool for recording brain activity and detecting potential neurological disorders.  The EEG is widely used in the diagnosis of neurological disorders such as epilepsy and the treatment of traumatic brain events such as strokes or injuries. In clinical settings, EEGs have several advantages over other brain imaging techniques--they are inexpensive, non-invasive, and directly measure neuron activity on the brain's cortical surface \cite{eegreview, NOACHTAR200922, Michel2019EEGSI}.

An array of $p$ electrodes on the scalp measures electrical signals produced by neuron activity. The forward process, which describes the causal relationship between current densities on the cortical surface of the brain and voltage detected by the electrodes, is modeled by the \emph{lead field matrix} $\bfL$.  Assuming that the dipole orientation at each source location is perpendicular to the surface of the brain, a scalar can represent the density at each source location, and thus $\bfL \in \bbR^{p\times n}$ with $n$ denoting the number of source locations. Hence, each column of $\bfL$ maps the \emph{current density} at time $t$, of \emph{source locations} $\bfx(t) \in \bbR^n$ to the voltage $\bfb(t) \in \bbR^p$ measured by the electrodes on the scalp \cite{Sarvas1987BasicMA, HUANG2016150, LFMBOOK}. We model the source locations on the cortical surface as vertices of a surface mesh and consider temporal dynamics. That is, at a fixed time $t$, we may describe the relationship by
\begin{equation}
    \bfb(t) = \bfL\bfx(t) + \bfvarepsilon(t),
    \label{equation:eq1}
\end{equation}
where $\bfvarepsilon(t)$ describes additive noise in the measuring process.  The resolution of the cortical surface mesh depends on $n$.  For clinical use, it is necessary to infer which regions of the cortical surface are active from the electrode readings, and thus an \emph{inverse problem} must be solved. Moreover, \emph{real-time} reconstruction of high-resolution neural brain activity is desirable.

For simplicity of notation, we assume an EEG records time dynamic data $\bfb(t)$ at discrete time points $t = 1,\ldots, T$. In short we write $\bfB = \begin{bmatrix}
    \bfb_1,\ldots,\bfb_T
\end{bmatrix} \in \bbR^{p\times T}$. 
Inference of the electrical activity at each source location for each time point, $t$, is commonly referred to as \emph{source localization} and we represent these values by $\bfX = \begin{bmatrix}
    \bfx_1,\ldots,\bfx_T
\end{bmatrix} \in \bbR^{n\times T}$.  The underlying \emph{inverse problem} is formulated as: given observations $\bfB$, lead field matrix $\bfL$ and assumption on the noise distribution $\bfvarepsilon(t)$, reconstruct the temporal brain activity on the cortical mesh $\bfX$.  As a variational inverse problem, this can be formulated, typically under Gaussian noise assumptions, as a least-squares problem
\begin{align}\label{eq:ls}
    \argmin_{\bfX\in\bbR^{n \times T}}  \ \sum_{i = 1}^T \norm[2]{\bfL\bfx_i - \bfb_i}^2 = \argmin_{\bfX \in \bbR^{n\times T}} \norm[\fro]{\bfL\bfX - \bfB}^2,
\end{align}
where $\norm[2]{\mdot}$ is defined as the $2$-norm and $\norm[\fro]{\mdot}$ denotes the Frobenius norm.  \Cref{eq:ls} does not consider a time-dependent dynamic solution and each $\bfx_i$ may be solved independently.

Typically, this inverse problem is significantly underdetermined because the number of locations, $n$, at which  estimates of the activity are required, far exceeds the number of EEG electrodes, i.e., $p\ll n$. At the same time, noise contributions are significant \cite{hadamard1923lectures, Invrev}. Regularization is required to stabilize the inverse solution.  Thus, we consider solving the \emph{regularized} inverse problem
\begin{align}
    \argmin_{\bfX \in \bbR^{n\times T}} \norm[\fro]{\bfL\bfX - \bfB}^2 + \calR(\bfX),
\end{align}
where $\calR\colon \bbR^n\to\bbR$ encodes prior knowledge as a regularizer, \cite{hansen2010discrete}.
Proposed methods for solving this inverse problem largely vary in how prior information is included through the regularization term.  One common choice is  Tikhonov regularization i.e., $\calR(\bfX) = \mu\|\bfX\|_\fro^2$ \cite{hansen2010discrete}. This serves as the foundation for \emph{standardized low-resolution brain electromagnetic tomography} ({\tt sLORETA}), where the inverse problem is solved using matrix pseudo inverses, and is a standard choice in practice \cite{pascual2002standardized}.  

Although these straightforward $\ell_2$-regularized approaches are very fast, they frequently overestimate the extent of activated regions \cite{142641, PASCUALMARQUI199449, gramfort2012mixed}.  To address these limitations, and to incorporate the neurological understanding that realistic brain activity would be localized to only a few locations \cite{Michel2019EEGSI}, $\ell_1$-based priors were introduced promoting sparsity on the reconstructed meshes. 
Setting $\calR(\bfX) = \norm[1, 1]{\bfX}$, where $\norm[1,1]{\bfC} = \norm[1]{\bfc}$, $\bfc = \vec{\bfC}$ denotes the vectorized form of matrix $\bfC$ and $\norm[1]{\mdot}$ is the vector 1-norm, is commonly referred to as \emph{LASSO regression} \cite{51791361-8fe2-38d5-959f-ae8d048b490d}. Due to the separability of the $\ell_1$-norm, each time point remains independent of other time points. Thus, while effective in constructing sparse signals, an $\ell_1$-regularization term alone fails to capture the time dynamics of the data in EEG source localization \cite{gramfort2012mixed, FRISTON20081104}. 

Other approaches have leveraged combinations of $\ell_1$- and $\ell_2$-norms to recover sparse reconstructions, while also enforcing smoothness in the reconstructed signal over time \cite{UUTELA1999173, gramfort2012mixed, OU2009932, HAUFE2008726, 8051087}. Models that use a linear combination of   
$\ell_1$- and $\ell_2$-regularization terms are often referred to as \emph{elastic net} regularization \cite{zou2005regularization}.  \emph{Generalized elastic net} regularized models are elastic net regularized models that include operators in the $\ell_1$- and $\ell_2$-terms. The discretized \emph{total-variation} operator ({\tt TV}), which approximates the derivative at each entry of the vector via a finite difference approximation, is a common choice \cite{Caselles2015}.  In imaging, this computes the sum of the absolute differences between each pixel and its adjacent pixels.  Pairing this with an $\ell_1$-norm often reduces noise in the reconstruction and maintains sharp edges \cite{RUDIN1992259}. In this work, we frame the EEG source localization problem as a generalized-elastic net regularized optimization problem and use total variation operators for both the spatial reconstruction and the time dynamics of the signal, hence removing the time-independence of each $\bfx_i$. \\

\noindent \emph{Approaches and Challenges. } Despite the observed advantages of elastic net approaches, corresponding algorithms are typically computationally expensive preventing real-time reconstructions due to the non-smooth $\ell_1$ term.  Solving the inverse problem is further challenging when working with dynamic datasets and high-resolution meshes since this results in data too large for $\ell_1$ solvers to handle \cite{boyd2011distributed, chung2022variable}--a high-resolution mesh may contain on the order of $10,\!000$ nodes measured across hundreds of time points.  In the context of EEG reconstruction, high-resolution meshes can be mitigated by dividing up the cortical mesh into different regions and performing reconstruction on the smaller problems.  However, this may result in inaccurate artifacts on the edges of these regions.  For clinical applications, an EEG algorithm that can quickly perform reconstructions, or even compute localizations in real-time is desirable. Unfortunately, due to the significant computational burden associated with elastic net approaches, reconstructions are currently confined to post-hoc analysis. Because of these challenges, computationally less demanding methods with quick turnaround are currently preferred over elastic net approaches \cite{pascual2002standardized}. Thus there is a need for an elastic net solver capable of handling larger data efficiently.\\

\noindent\emph{Contributions.} Our work has several key contributions: 
Our first contribution is phrasing the EEG problem as a large-scale variational inverse problem regularized by graph total variation. This extends existing approaches and enables high-resolution inversion \cite{Caselles2015}. To address this large-scale inverse problem, we introduce a time-dynamic \emph{variable-projected augmented Lagrangian} ({\tt VPAL}) method. This method builds upon the approaches in \cite{chung2022variable}, which are limited by specific gradient updates. Our method offers greater flexibility in the update steps, leading to improved computational efficiency and outperforming existing state-of-the-art approaches. We present a novel convergence proof for the {\tt VPAL} approach on a non-restrictive class of functions, applicable to broader settings including nonlinear forward models. For EEG applications, we introduce a sequential version of our {\tt VPAL}-based method, where reconstruction at each time point is accelerated by using information from the previous time point and is suitable for near real-time data reconstruction.  Numerical comparisons against existing state-of-the-art methods, such as {\tt sLORETA}, are also provided illustrating the benefits of our approach.\bigskip

\noindent\emph{Structure.} This work is organized as follows. \Cref{section:background} details the formulation of the EEG inverse problem and 
reviews existing state-of-the-art approaches applicable in this context.  \Cref{sec:vpal} introduces our variable projected approach and establishes the theoretical convergence properties of our proposed method. \Cref{section:numExp} presents experimental results that demonstrate the time efficiency and scalability of our method for large datasets.  Here, we also discuss implementation details for our sequential approach, and illustrate real-time reconstruction using this method.  Finally, \Cref{section:Discussion} and \Cref{section:ConclusionsandOutlook} discuss the strengths and weaknesses of our approach and directions for future research.

Note, the approaches presented here are not restricted to the use case of EEG localization.  While we focus on source localization, the operators and algorithms that we propose apply to a wider class of (graph) inverse problems, such as graph-inpainting, anomaly detection, and other graphical tomography problems, see \cite{bianchi2021graphlaplacianimagedeblurring}.

\section{Background and Problem Setup}\label{section:background}

We consider a general elastic net model with an $\ell_1$ regularization term on the spatial data and an $\ell_2$ regularization term for temporal data. This model is founded on the common assumption that a sparse set of active brain regions generates the measured electrode signals and noise contributions are significant \cite{Michel2019EEGSI}. Sparsity in the total variation of mesh values is imposed to suppress non-pertinent sharp changes in activation values, thereby isolating significant signals and mitigating noise. Time regularization on the other hand intends to reduce noise and enforce smooth changes in activation over time, based on the assumption that changes in the activation values are smooth in time assuming a fine enough discretization of time points.  An $\ell_2$ term is therefore suitable for time regularization due to the smoothing effect it will have on the result \cite{Diffellah2021ImageDA}.  For the remainder of this paper, sparsity is used to refer to sparsity in the total variation of the mesh as opposed to the number of non-zero entries of the mesh.   

Let the mesh of the cortical surface be represented as a temporal graph network consisting of $n$ nodes and $m$ edges corresponding to the vertices and edges in the mesh respectively with the value of each node dependent on time point $t$. The resulting optimization problem can be formulated as
\begin{equation} \label{equation:objective}
   \argmin_{\bfX\in\bbR^{n\times T}} \quad f(\bfX) =  \thf\norm[\fro]{\bfL\bfX - \bfB}^2 + \tfrac{\lambda^2}{2}\norm[\fro]{\bfX\bfD_1}^2 + \mu\norm[1,1]{\bfD_2\bfX}, 
\end{equation}
with $\lambda,\mu >0$ The matrix $\bfD_1 \in \bbR^{T\times T-1}$ is the finite time-difference operator and $\bfD_2 \in \mathbb{R}^{m \times n}$ is the finite-difference operator \cite{Caselles2015} on the cortical mesh. Note while \Cref{eq:ls} can be formulated as $T$ independent optimization problems,  with the introduction of the time regularization term we established a dependence between the solutions $(\bfx_i)_{i=1}^T$. 

The operation with $\bfD_1$ can be represented as
\begin{align}
    \bfX\bfD_1 = \begin{bmatrix}
    \bfx_1 - \bfx_2, & \bfx_2 - \bfx_3,& \ldots, & \bfx_{T-1} - \bfx_T\end{bmatrix},
\end{align}
where we omit the time-frequency scaling constant, noting that it is absorbed into the regularization weighting parameter $\lambda$.
 
The {\tt graphTV} operator, $\bfD_2$, requires more attention.  Using graph theory notation, the $\ell_1$-norm of the total variation of a graph, $G = (V, E)$, where $V$ represents the vertices and $E$ the edges, is given as
\begin{align}
    \norm[1]{{\tt graphTV}(G)} = \sum_{(u, v) \in E} |w_{(u, v)}(u.\textnormal{value} - v.\textnormal{value})|
\end{align}
with graph $G$ either directed or undirected, and $w_{(u, v)}$ the weight of the edge $(u, v)$.  We may state the {\tt graphTV} operation for the EEG problem as
\begin{align}\label{eq:graphTV}
    \bfD_2\bfX = \bfW(\bfP_1\bfX - \bfP_2\bfX),
\end{align}
where $\bfW \in \bbR^{m \times m}$ is a diagonal matrix consisting of the weights of the edges. 
Further, $\bfP_1, \bfP_2 \in \bbR^{m \times n}$ are permutation matrices that index $\bfX$ to facilitate the subtraction of each node's value from only its immediate neighbors.  This definition of {\tt graphTV} is analogous to total variation ({\tt TV}) for images, see \Cref{fig:TV_illustration}.

\begin{figure}[htb!]
    \centering
    \includegraphics[width=0.7\linewidth]{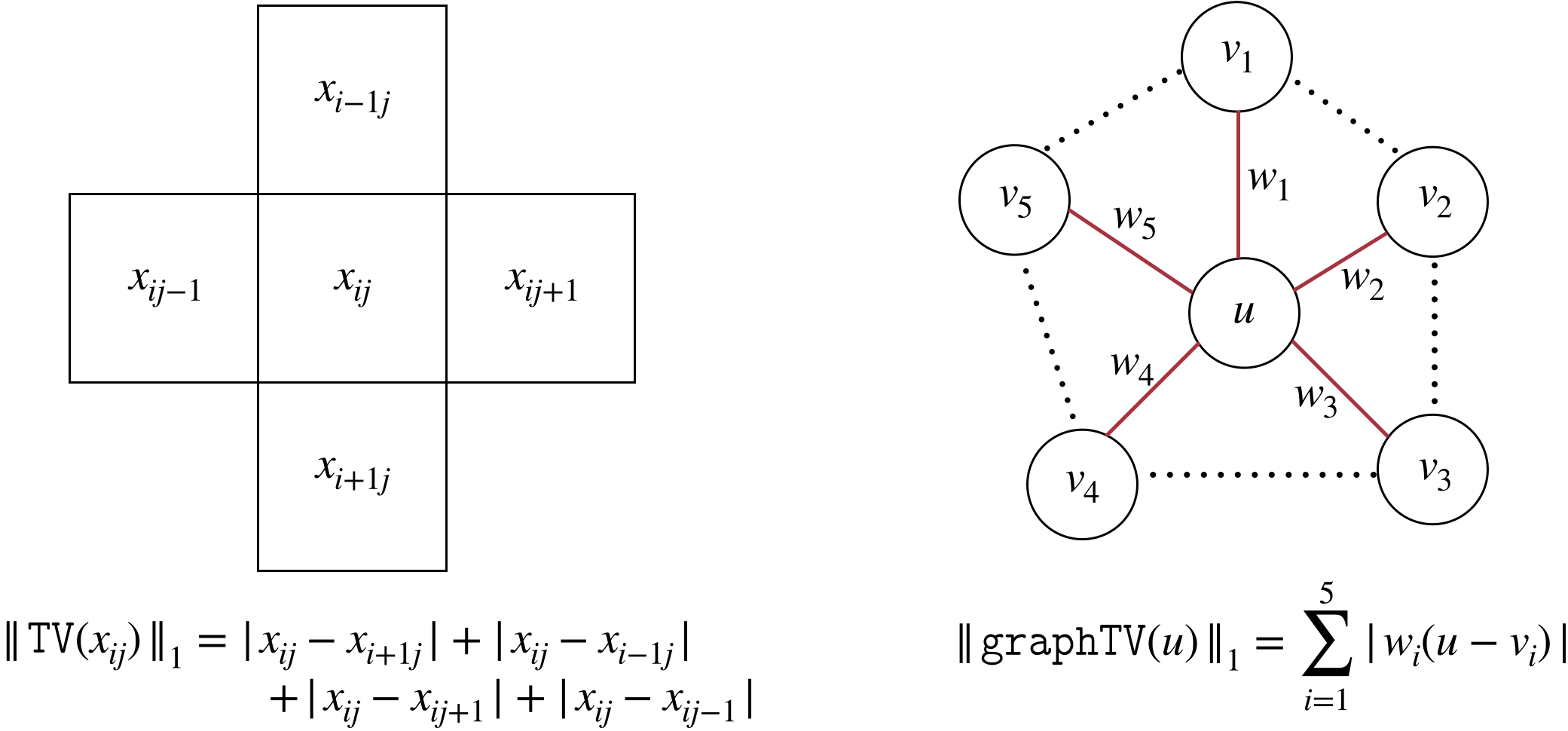}
    \caption{Illustration of the difference between the standard total variation operator  {\tt TV} for images in the left panel and our implementation of {\tt graphTV} on the right.  Note that in the image, the total variation is calculated as the difference between the middle pixel, $x_{ij}$ and adjacent pixels.  This is defined analogously for graphs where adjacency is interpreted as having an edge instead.}
    \label{fig:TV_illustration}
\end{figure}

One common approach for solving elastic net regularized problems involves splitting methods \cite{Liu2023}.  These methods introduce an auxiliary variable to decouple the regularization term from the rest of the objective function. This reformulation allows the problem to be expressed as an augmented Lagrangian, which combines the original objective with a penalty term that enforces consistency between the original variable and the newly introduced surrogate variable.  Then, an alternating optimization procedure is employed updating the primary variable, the auxiliary variable and the Lagrange multipliers.  This procedure iteratively minimizes the augmented Lagrangian with respect to each variable while holding the other fixed.  This splitting strategy often simplifies the optimization process, as each subproblem becomes easier to solve than the original, non-split problem.

The \emph{Alternating Direction Method of Multipliers} ({\tt ADMM}) is a standard approach to solving general elastic net regularized problems \cite{boyd2011distributed}.  In general, {\tt ADMM} applies to separable optimization problems, subject to linear constraints i.e.,
\begin{align}\label{eq:splitting1}
    \argmin_{\bfx, \bfy} \; & q(\bfx) + r(\bfy) \\
    \textrm{s.t.} & \; \bfQ\bfx + \bfR\bfy = \bfv, \label{eq:splitting2}
\end{align}
with $\bfx \in \bbR^{n_x}, \bfy \in \bbR^{n_y},  \bfv \in \bbR^{n_p}, \bfQ \in \bbR^{n_p \times n_x}, \bfR \in \bbR^{n_p \times n_y}, $ and $q:\bbR^{n_x} \rightarrow \bbR, r:\bbR^{n_y} \rightarrow \bbR$, both convex.  Due to the constraints, an optimal solution can be found through the minimization of the augmented Lagrangian which has the form
\begin{align}\label{eq:auglag}
    h(\bfx, \bfy; \bfc) = \varphi(\bfx, \bfy; \bfc) +  r(\bfy),
\end{align}
where $\bfc \in \bbR^{n_p}$ denotes the Lagrange multipliers and  $\varphi:\bbR^{n_x} \times \bbR^{n_y} \rightarrow \bbR$ denotes the terms of the augmented Lagrangian that do not soley depend on $\bfy$.  The {\tt ADMM} algorithm is then given by the three-step update procedure summarized in \Cref{alg:admm}.  Here, $\eta > 0$ denotes the penalty parameter for the augmented Lagrangian \cite{cfc5dc07425343f08f3c8ee5ae8f7ddc}.

\begin{algorithm}[htb!]
\caption{{\tt ADMM}}\label{alg:admm}
    \begin{algorithmic}[1]
    \State \textbf{input} $\varphi(\bfx, \bfy; \bfc)$, $r(\bfy)$, $\eta$, $\bfx_0$, $\bfR$, $\bfQ$, $\bfv$ 
   \State initialize $\bfx_0, \bfy_0={\bf0}, \bfc_0 = \bf0$ and set $k = 0$
        \While{not converged}
        \State $\bfx_{k+1} = \argmin_\bfx \varphi(\bfx, \bfy_k; \bfc_k)$ \label{ln:xadmm}
        \State $\bfy_{k+1} = \argmin_{\bfy} \varphi(\bfx_{k+1}, \bfy; \bfc_k) + r(\bfy)$\label{ln:yadmm}
        \State $\bfc_{k+1} = \bfc_k + \eta^2(\bfQ\bfx_{k+1} + \bfR\bfy_{k+1} - \bfv)$
        \State $k = k + 1$
        \EndWhile
        \State \textbf{output} $\bfx_k$
    \end{algorithmic}
\end{algorithm}

{\tt ADMM} addresses \Cref{equation:objective} by introducing an auxiliary variable $\bfY = \bfD_2\bfX$ to decompose the problem into a separable problem subject to the constraint $\bfY = \bfD_2\bfX$, see
\cite{GLW1975, GABAY197617, boyd2011distributed}. Letting $\bfC$ denote the Lagrange multipliers, the augmented Lagrangian for our particular problem is given as
\begin{align}\label{equation:aug_lag}
   \calL_{\textnormal{aug}}(\bfX, \bfY ; \bfC) = \thf \norm[\fro]{\bfL\bfX-\bfB}^2 + \tfrac{\lambda^2}{2}\norm[\fro]{\bfX\bfD_1}^2 + \mu\norm[1,1]{\bfY} + \tfrac{\eta^2}{2}\norm[\fro]{\bfD_2\bfX-\bfY+\bfC}^2 - \tfrac{\eta^2}{2}\norm[\fro]{\bfC}^2.
\end{align}
Note, since $\tfrac{\eta^2}{2}\norm[\fro]{\bfC}^2$ is independent of $\bfX$ and $\bfY$ in \Cref{equation:aug_lag} we omit it when discussing updates for $\bfX$ and $\bfY$. Additionally, \Cref{equation:objective} \Cref{eq:auglag} are given in terms of matrices, whereas \Cref{eq:splitting1} and \Cref{eq:splitting2} are in terms of vectors.  Through Kronecker-product identities we can vectorize our objective function to fit \Cref{eq:splitting1} and \Cref{eq:splitting2}--see \Cref{appendix:vectorizedforms}.  Following \Cref{alg:admm}, and rewriting updates in matrix-form, the algorithm proceeds as follows.  Indexing each iterate by $k$, the introduction of surrogate variable $\bfY$ results in the splitting of the minimization of $\calL_{\textnormal{aug}}(\bfX, \bfY; \bfC)$ into a repeated linear least squares problem 
\begin{align}\label{equation:genericxupdate}
    \bfX_{k+1} \in \argmin_{\bfX}\thf \norm[\fro]{\bfL\bfX - \bfB}^2 + \tfrac{\lambda^2}{2}\norm[\fro]{\bfX\bfD_1}^2 + \tfrac{\eta^2}{2}\norm[\fro]{\bfD_2\bfX - \bfY_k + \bfC_k}^2,
\end{align}
followed by the minimization of the terms dependent on $\bfY$, given by the soft-thresholding update
\begin{equation}\label{equation:shrinage}
  \bfY_{k+1} = \textnormal{sign}(\bfD_2\bfX_{k+1} + \bfC_k)\hadamard \text{ReLU}\left(|\bfD_2\bfX_{k+1} + \bfC_k| - \bfN\right).
\end{equation}
Here $|\mdot|$ denotes elementwise absolute value,  $(\bfN)_{ij} = \frac{\mu}{\eta^2}$, $\hadamard$ denotes the Hadamard product, and $\text{ReLU}(\mdot)$ is the rectified linear unit. The Lagrange multipliers are updated according to 
\begin{equation}
\bfC_{k+1} = \bfC_k + \eta^2(\bfD_2 \bfX_{k+1} - \bfY_{k+1}).
 \end{equation}

This method is bottlenecked by the least-squares update step \Cref{equation:genericxupdate}, which requires solving a system for $n\cdot T$ unknowns at each iteration \cite{boyd2011distributed, chung2022variable}. For large-scale problems, methods such as {\tt LSQR} or other generalized Krylov subspace approaches may be utilized, \cite{10.1145/355984.355989, fastADMM}. Nevertheless, solving \Cref{equation:genericxupdate} remains the computationally expensive part.

EEG problems that include dynamic data or high-resolution meshes quickly become intractable for {\tt ADMM} due to this bottleneck. A common alternative to {\tt ADMM} is the \emph{Fast Iterative Threshold-Shrinking Algorithm} ({\tt FISTA}), see \cite{doi:10.1137/080716542}.  This approach is built around the minimization of the proximal at each iteration which has a closed-form solution for standard elastic net problems resulting in fast and efficient updates.  However, {\tt FISTA} loses this efficiency for generalized elastic net problems due to the inclusion of an operator in the $\ell_1$-norm, as is the case in \Cref{equation:objective} with $\bfD_2$.  The minimization of the proximal subproblem no longer has a closed-form solution, and thus must be solved iteratively \cite{manoel2018approximatemessagepassingconvexoptimization}.  As another downside, {\tt FISTA} requires the estimation of the Lipschitz constant of the objective function, which poses an additional challenge. For comparison purposes, we include an implementation of {\tt ADMM} and a variation of {\tt FISTA} where each update is solved iteratively using {\tt VPAL}, see \Cref{appendix:algo_options} and \Cref{appendix:admm_update}.

\section{Variable Projection Augmented Lagrangian Method ({\tt VPAL})}\label{sec:vpal}

This section details our proposed Variable Projected Augmented Lagrangian (VPAL) method.   \Cref{sub:vpalncg} presents the derivation and specifics of the method using nonlinear conjugate gradient updates.  Convergence results for a broader class of objective functions are provided in  \Cref{section:theory}.

\subsection{{\tt VPAL} with Nonlinear Conjugate Gradient Update}\label{sub:vpalncg}
In light of the challenge discussed in the previous section, we propose to use a \emph{Variable Projected Augmented Lagrangian} ({\tt VPAL}) approach which replaces the computationally expensive linear solve in {\tt ADMM} with an efficient variable projected-update step suitable for inference on massive graphs \cite{chung2022variable}. For analysis purposes and to ease notation, we consider the equivalent vectorized form ($\bfx = \vec{\bfX}$) of our objective function \Cref{equation:objective} and augmented Lagrangian \Cref{equation:aug_lag} (see derivations in \Cref{appendix:vectorizedforms}). Using $\otimes$ to denote the Kronecker product, then  for $\bfA = \begin{bmatrix}
    \bfI_T \otimes \bfL \\ \lambda(\bfD_1^\top \otimes \bfI_n)
\end{bmatrix} \in \bbR^{Tp + (T-1)n \times Tn}$ , $\bfw = \begin{bmatrix}
    \bfb \\ \textbf{0}_{n(T-1)}
\end{bmatrix} \in {\bbR^{Tp+n(T-1)}}$, and 
    $\tilde{\bfD}_2 = \bfI_T \otimes \bfD_2 \in \bbR^{mT \times nT},$
 we have

\begin{align}
f(\bfx) = \thf\norm[2]{\bfA\bfx - \bfw}^2 + \mu \norm[1]{\tilde \bfD_2 \bfx},
\end{align}
and \begin{align}\label{equation:vectorizedauglag}
       \calL_{\textnormal{aug}}(\bfx, \bfy ; \bfc) =  \thf\norm[2]{\bfA\bfx - \bfw}^2 + \tfrac{\eta^2}{2}\norm[2]{\tilde{\bfD}_2\bfx - \bfy + \bfc}^2 + \mu \norm[1]{\bfy} - \tfrac{\eta^2}{2}\norm[2]{\bfc}^2.
\end{align}

We denote the gradient of $\calL_{\rm aug}(\bfx, \bfy)$ with respect to $\bfx$, $\nabla_\bfx \calL_{\rm aug}(\bfx, \mdot) = \bfg(\bfx, \mdot)$, and $\nabla_\bfx \calL_{\rm aug}(\bfx^j, \bfy) = \bfg^j(\bfy)$.  We focus specifically on \Cref{ln:xadmm} and \Cref{ln:yadmm} of \Cref{alg:admm}.  Since $\bfc$ is fixed during the execution of these two lines we suppress dependence on $\bfc$, and consider
\begin{align}\label{eq:fjoint}
    f_{\textnormal{joint}}(\bfx, \bfy) = \thf\norm[2]{\bfA\bfx - \bfw}^2 + \tfrac{\eta^2}{2}\norm[2]{\tilde{\bfD}_2\bfx - \bfy + \bfc}^2 + \mu \norm[1]{\bfy}.
\end{align}
Minimizing this function $f_{\text{joint}}$ with respect to $\bfy$ has the closed-form solution 
\begin{align}\label{equation:l1terms}
    \bfZ(\bfx) = \argmin_{\bf{y}} \tfrac{\eta^2}{2}\norm[2]{\tilde\bfD_2\bfx - \bfy + \bfc}^2 + \mu \norm[1]{\bfy}.
\end{align}
Here $\bfZ: \bbR^{nT} \rightarrow \bbR^{mT}$ is the vectorized continuous mapping of the soft thresholding step $\bfy$, see \Cref{equation:shrinage}
\begin{align}\label{eq:shrink}
    \bfZ(\bfx) = \sign{\tilde \bfD_2\bfx + \bfc} \hadamard \text{ReLU}\left(|\tilde\bfD_2\bfx + \bfc| - \tfrac{\mu}{\eta^2}\right).
\end{align}
Using the closed form expression for optimal $\bfy$, we project the joint problem onto the space of optimal $\bfy$ and consider the projected problem

\begin{equation}\label{eq:proj}
    f_{\rm proj}(\bfx) = \thf\norm[2]{\bfA\bfx - \bfw}^2 + \tfrac{\eta^2}{2}\norm[2]{\tilde{\bfD}_2\bfx - \bfZ(\bfx) + \bfc}^2 +\mu \norm[1]{\bfZ(\bfx)},
\end{equation}
which has implicit dependence on $\bfy$ through $\bfZ(\bfx)$.  

In the variable projected approach, the computationally expensive solve to update $\bfx$ is replaced by an iterative line search, solving for $\bfx$ that minimizes the projected problem $f_{\text{proj}}$ \cite{doi:10.1137/0710036, chung2022variable}. This approach is based on the following intuition: Away from the coordinate axes, we have that the $\bfy$-dependent terms of $\calL_{\rm aug}$ are differentiable and convex.  Thus 
\begin{align}
    \eta^2 (\tilde \bfD_2\bfx + \bfc - \bfy) - \mu\, \sign{\bfy} = {\bf0}
\end{align}
is a necessary condition on $\bfy$ to minimize $\calL_{\rm aug}$ if $\bfx$ is treated as a constant. We also observe the following property of $f_{\textnormal{proj}}(\bfx)$ away from points where the function is not differentiable.  Let $\calJ_\bfZ(\bfx)$ denote the Jacobian of $\bfZ(\bfx)$.  Then,

\begin{multline}
 \nabla_\bfx\left(f_{\rm proj}(\bfx)\right) =  \bfA\t(\bfA\bfx -\bfw) +  \tfrac{\eta^2}{2} \nabla_\bfx\left(\norm[2]{\tilde{\bfD}_2\bfx - \bfZ(\bfx) +\bfc}^2 \right) + \mu \nabla_\bfx \left( \norm[1]{\bfZ(\bfx)}\right)\\  
   = \bfA\t(\bfA\bfx -\bfw) + \eta^2(\tilde{\bfD}_2^\top\tilde{\bfD}_2\bfx + \tilde{\bfD}_2^\top\bfc - \tilde{\bfD}_2^\top\bfZ(x) + \calJ_\bfZ(\bfx)^\top\tilde{\bfD}_2\bfx + \calJ_\bfZ(\bfx)^\top \bfZ(\bfx) +  \calJ_\bfZ(\bfx)^\top \bfc) + \mu\,\calJ_\bfZ(\bfx)^\top\sign{\bfZ(\bfx)}\\
   = \bfA\t(\bfA\bfx -\bfw) + \eta^2 \left(\tilde{\bfD}_2\t\tilde{\bfD}_2\bfx + \tilde{\bfD}_2\t \bfc -\tilde{\bfD}_2\t\bfZ(\bfx)\right) +\calJ_\bfZ(\bfx)\t \left(\underbrace{\eta^2\left(-\tilde{\bfD}_2\bfx +\bfZ(\bfx) - \bfc\right) + \mu\,\sign{\bfZ(\bfx)}  }_{=\bf0}\right) \\
   = \bfA\t(\bfA\bfx -\bfw) + \lambda^2 \tilde{\bfD}_2\t\left(\tilde{\bfD}_2\bfx + \bfc -\bfZ(\bfx)\right). 
\end{multline}
Thus, away from points where $f_{\rm proj}$ is not differentiable, the gradient of the projected problem is the same as if $\bfZ(\bfx)$ were constant in $\bfx$. That is, in the projected space of optimal $\bfy$, we need not consider the dependence of $\bfy$, and can solve this as a univariate optimization problem.  The variable projected approach reflects this intuition by updating $\bfx$ through a line search as if $\bfy$ were constant and then updates $\bfy$ to be optimal at the new location. Ignoring the non-differentiability will cause difficulties in an iterative optimization routine and more rigorous analysis is given in \Cref{section:theory}. The full algorithm is given in \Cref{alg:vpal} and with appropriate assumptions, is shown to converge for the non-differentiable EEG objective function.

The work \cite{chung2022variable} introduced variable projection for $\ell_1$-regularized inverse problems and focused on using the negative gradient direction for the $\bfx$ update and provides the corresponding convergence analysis for this case.  We propose a generalization, where $\bfx$ is updated by a \emph{nonlinear conjugate line search} \cite{10.1093/comjnl/7.2.149, cfc5dc07425343f08f3c8ee5ae8f7ddc} since the nonlinear conjugate gradient update often performs better in comparison to the steepest descent in smooth unconstrained optimization problems. We refer to this algorithm using the same name as \cite{chung2022variable}, the \emph{variable projected augmented Lagrangian} ({\tt VPAL}), since we only modify the line search in each iteration.  Note that the particular EEG problem at hand fits within the problem setup \Cref{eq:splitting1,eq:splitting2,eq:auglag}, specifically we have $\varphi(\bfx, \bfy; \bfc) = \thf\norm[2]{\bfA\bfx - \bfw}^2 + \tfrac{\lambda^2}{2}\norm[2]{\tilde{\bfD}_2\bfx - \bfy + \bfc}^2, r(\bfy) = \mu\norm[1]{\bfy}, \calZ(\bfx) = \bfZ(\bfx), \bfg(\bfx, \bfy) = \vec{\bfL^\top(\bfL\bfX - \bfB) + \lambda^2(\bfX\bfD_1 \bfD_1^\top) +  \eta^2\bfD_2^\top(\bfD_2\bfX - \bfY + \bfC)}, \bfQ = \tilde{\bfD}_2, \bfR = -\bfI_{mT}, \bfv = {\bf0}$.  For details on the gradient computation, see \Cref{appendix:vpalncgupdate}.

\begin{algorithm}[htb!]
\caption{{\bf {\tt VPAL}}}\label{alg:vpal}
    \begin{algorithmic}[1]
    \State \textbf{input} $\varphi(\bfx, \bfy; \bfc), r(\bfy), \bfg(\bfx, \bfy)$, $\calZ(\bfx)$, $\bfx_0$, $\eta$,$\bfQ$, $\bfR$,  $\bfv$
    \State initialize $\bfx_0, \bfc_0={\bf 0}, \bfy_0={\bf0}$, $\bfs_0 = -\bfg_0(\bfy_0)$  and set $k = 0$
    \While{not converged}\label{ln:outer_start_vpalncg}

    \State set $j = 0$, $\bfx^{(0)} = \bfx_k$, $\bfy^{(0)} = \bfy_k$, $\bfs^{(0)} = -\bfg(\bfx_k, \bfy_k)$ 
        \While{not converged}\label{ln:inner_start_vpalncg}\Comment{Conjugate Gradient Line Search}
        \State Compute projected step length, $\alpha^{(j)}$\label{ln:linesearch}
        \State Update $\bfx^{(j+1)} = \bfx^{(j)} - \alpha^{(j)} \bfs^{(j)}$\label{ln:xupdate}
        \State Update $\bfy^{(j+1)} = \calZ({\bfx^{(j+1)}})$
        \State Calculate $\bfg^{(j+1)} (\bfy^{(j+1)})$
        \State Compute $\beta^{(j+1)}$ \Comment{Conjugate Gradient $\beta$ Computation} \label{ln:betaselection}
        \State $\bfs^{(j+1)} = -\bfg^{(j+1)}(\bfy^{(j+1)}) + \beta^{(j+1)}\bfs^{(j)}$
        \State $j = j+1$
    \EndWhile\label{ln:inner_end_vpal}
    \State set $\bfx_{k+1} = \bfx^{(j)}$ and $\bfy_{k+1} = \bfy^{(j)}$
     \State $\bfc_{k+1} = \bfc_k + \eta^2(\bfQ\bfx_{k+1} - \bfR\bfy_{k+1}-\bfv)$ \Comment{Lagrange Multiplier Update}
    \State $k = k+1$
    \EndWhile\label{ln:outer_end_vpal}
    \State \textbf{output} $\bfx_k$ 
    \end{algorithmic}\label{alg:vpaldynamic}
\end{algorithm}

We denote the conjugate gradient search direction as $\bfs^{(j)}$ and consider the line search update routine in \Cref{ln:xupdate} of \Cref{alg:vpal}:
\begin{align}
    \bfx^{(j+1)} = \bfx^{(j)} - \alpha^{(j)} \bfs^{(j)}.
\end{align}

To perform an update step of $\bfx$ via line search, we require an appropriate step size selection $\alpha^{(j)}$. There are various strategies to obtain this step size. We consider two approaches. 

In the first approach, we use the optimal step size. Due to the setup of this problem, this optimal step size can be computed inexpensively with the main computational cost being the evaluation of the soft thresholding. We solve the following 1-dimensional optimization problem
\begin{multline}
\argmin_{\alpha} \ f_\textnormal{proj}(\bfx^{(j)} - \alpha\bfs^{(j)}) = \argmin_{\alpha} \thf \norm[2]{\bfA(\bfx^{(j)} - \alpha \bfs^{(j)})  - \bfw}^2 +  \\
+\tfrac{\eta^2}{2}\norm[2]{\tilde{\bfD}_2(\bfx^{(j)} - \alpha \bfs^{(j)}) - \bfy^{(j)} + \bfc_k}^2  + \mu \norm[1]{\bfZ(\bfx^{(j)} - \alpha\bfs^{(j)})}.
\end{multline}
Hence, efficient minimization can be attained through the application of a simple one-dimensional iterative solver, where previous $\alpha^{(j)}$'s may be utilized to accelerate convergence.

Another approach is to ignore the nonlinearity of the optimization problem \Cref{equation:genericxupdate} and use the linearized optimal step size corresponding to the linear conjugate gradient method. Empirically we observe that this linearization is sufficient and gives good estimates, and we can always fall back to the the optimal stepsize in case this one fails. To obtain the linearized step size, we ignore the nonlinear shrinkage term in \Cref{eq:shrink} and instead solve
\begin{multline}
\argmin_{\alpha} \ f_\textnormal{proj}(\bfx^{(j)} - \alpha\bfs^{(j)}) = \thf \norm[2]{\bfA(\bfx^{(j)} - \alpha \bfs^{(j)})  - \bfw}^2 
+\tfrac{\eta^2}{2}\norm[2]{\tilde{\bfD}_2(\bfx^{(j)} - \alpha \bfs^{(j)}) - \bfy^{(j)} + \bfc_k}^2.
\end{multline}
In this case, we can obtain the optimal closed-form solution
\begin{align}
    \alpha^{(j)} = \frac{\bfs^{(j)\top}\vec{\bfL^\top\bfL\bfX^{(j)} - \bfL^\top\bfB + \eta^2\bfD_2^\top\bfD_2\bfX + \bfD_2^\top\bfD_2(\bfC_k - \bfY^{(j)}) + \lambda^2\bfX^{(j)}\bfD_1\bfD_1^\top}}{\bfs^{(j)\top}\vec{\bfL^\top\bfL\bfS^{(j)} + \eta^2\bfD_2^\top \bfD_2\bfS^{(j)} + \lambda^2 \bfS^{(j)}\bfD_1\bfD_1^\top}}.
\end{align}

There are several options for the computation of $\beta^{(j)}$ in the nonlinear conjugate gradient update and the interested reader is referred to \cite{cfc5dc07425343f08f3c8ee5ae8f7ddc}.  In our implementation, we use a combination of the Fletcher-Reeves (FR) \cite{10.1093/comjnl/7.2.149} and Polak-Ribiére (PR) \cite{PRupdate} formulas which have empirically demonstrated robust performance on a variety of problems.  The FR $\beta$ update is given by
\begin{align}
    \beta^{(j+1)}_{\textnormal{FR}} = \frac{\bfg^{j+1}(\bfy^{(j+1)})^\top \bfg^{j+1}(\bfy^{(j+1)})}{\bfg^{j}(\bfy^{(j)})^\top \bfg^{j}(\bfy^{(j)})},
\end{align} 
and the PR update by
\begin{align}\label{eq:PR}
        \beta^{(j+1)}_{\textnormal{PR}} = \frac{\bfg^{j+1}(\bfy^{(j+1)})^\top (\bfg^{j+1}(\bfy^{(j+1)}) -\bfg^{j}(\bfy^{(j)}) )}{\bfg^{j}(\bfy^{(j)})^\top \bfg^{j}(\bfy^{(j)})}.
\end{align}
A more robust computation of $\beta^{(j+1)}$ is given as follows
\begin{align}\label{eq:betac}
    \beta^{(j+1)} = \begin{cases}
        -\beta^{(j+1)}_{\textnormal{FR}}, & \textnormal{for }  \; \beta^{(j+1)}_{\textnormal{PR}} < -\beta^{(j+1)}_{\textnormal{FR}},\\[1.5ex]
        \beta^{(j+1)}_{\textnormal{PR}}, & \textnormal{for } \; \big|\beta^{(j+1)}_{\textnormal{PR}}\big| \leq \beta^{(j+1)}_{\textnormal{FR}}, \\[1.5ex]
        \beta^{(j+1)}_{\textnormal{FR}}, & \textnormal{for } \; \beta^{(j+1)}_{\textnormal{PR}} > \beta^{(j+1)}_{\textnormal{FR}}.
    \end{cases}
\end{align}

Although we prove convergence for the $\beta_{\rm PR}$ update with a particular line search discussed in \Cref{section:theory}, implementation with the linearized step size selection and the $\beta$ computation given in \Cref{eq:betac} performed better numerically. 

As noted in \cite{chung2022variable}, practical implementations of \Cref{alg:vpal} only perform a few iterations of the inner loop.  For our implementation, we perform 2 iterations of \Cref{ln:inner_start_vpalncg} through \Cref{ln:inner_end_vpal}.  Note that reducing this to one iteration of the inner loop exactly recovers the ``standard" {\tt VPAL} presented in \cite{chung2022variable}.

Note that {\tt VPAL} is not simply {\tt ADMM} with a single (Krylov) $\bfx$ update/inexact solve.  {\tt VPAL} updates the \emph{nonlinearly projected problem} \Cref{eq:proj}, hence the name variable projection.  The inherent alternating direction approach of {\tt ADMM} means the variability in $\bfy$ is not directly incorporated into the $\bfx$ update. {\tt VPAL}, however, directly addresses the projected problem, allowing $\bfy$ to affect and guide the $\bfx$ update via the $\bfy = \bfZ(\bfx)$ relationship.

\subsection{{\tt VPAL} Convergence Analysis}\label{section:theory}

This section establishes the convergence of our proposed {\tt VPAL} method in the more general setting of \Cref{eq:splitting1}, \Cref{eq:splitting2} and \Cref{eq:auglag}, beyond the specific application to EEG and extends the proof given in \cite{chung2022variable}.  Specifically, we generalize our analysis to a convex and potentially nonlinear forward operator. To provide this convergence analysis,  we first establish conditions for convergence of our algorithm for a more general class of functions. Using these conditions, we prove Lemmas used in the Proof of our main convergence result, \Cref{thm:convergence}.  Finally, we conclude by establishing convergence criteria for our EEG source localization model.  

We consider the setup of the problem as stated in \Cref{section:background}, and focus on minimizing \Cref{eq:auglag}.  The convergence of the outer loop of \Cref{alg:vpal} is well established in {\tt ADMM} literature for convex $q(\cdot)$ and $r(\cdot)$ \cite{chung2022variable, boyd2011distributed}, thus it is sufficient to show that \Cref{ln:inner_start_vpalncg} -- \Cref{ln:inner_end_vpal} of \Cref{alg:vpal} solve
\begin{align}
    (\bfx_{k+1}, \bfy_{k+1}) = \argmin_{\bfx, \bfy} \ \ h(\bfx, \bfy; \bfc_k) = \varphi(\bfx, \bfy; \bfc_k) + r(\bfy).
\end{align}
Since $\bfc_k$ is held constant held constant in the inner loop, we suppress it for ease of notation. The convergence theory that we present for this subproblem closely follows \cite{zhou2013short, zhou2021class}, and the interested reader is referred to these references for further details. \\

\noindent\emph{Assumptions and Settings. }We now assume that $\varphi(\bfx, \bfy)$ is smooth and strictly convex in both $\bfx$, and $\bfy$, and $r(\cdot)$ is convex, but possibly non-smooth. In this case, we have $h(\bfx, \bfy)$ strictly convex, with unique minimizer denoted $(\hat\bfx, \hat\bfy)$.  We make the additional assumption that for all $\bfx$, we have access to a closed-form solution to the following minimization
\begin{align}\label{eq:yupdate}
    \calZ(\bfx) = \argmin_{\bfy \in \bbR^{n_y}}\;   h(\bfx, \bfy).
\end{align}
Denoting $\partial_{\bfy}r(\bfy)$ as the subgradient of $r(\bfy)$, if
\begin{align}
    {\bf 0} \in \nabla_{\bfx, \bfy}\varphi(\hat\bfx, \hat\bfy) + \partial_{\bfy}r(\hat\bfy),
\end{align}
then $(\hat\bfx, \hat\bfy)$ is a stationary point. Due to the strict convexity of $h(\bfx, \bfy)$ a sufficient and necessary condition for $(\hat\bfx, \hat\bfy)$ to be a minimizer is that it is a stationary point.

To prove convergence of our algorithm, the following choices are made for the step size selection and $\beta$ computation in \Cref{ln:betaselection}.  We adopt the step size selection routine proposed in \cite{zhou2013short}.  That is, given $\delta > 0, \rho \in (0, 1), \epsilon > 0$ and a positive sequence $\{\epsilon_j\}_{j = 1}^\infty$ satisfying 
\begin{align}\label{equation:eps}
    \sum_{j = 0}^\infty \epsilon_j \leq \epsilon < \infty,
\end{align}
we use a backtracking line search $\alpha^{(j)} = \max \{1, \rho^1, \rho^2, \dots \}$ satisfying  
\begin{align}\label{equation:alpha_comp}
    \varphi(\bfx^{(j)} + \alpha^{(j)}\bfs^{(j)}) \leq \varphi(\bfx^{(j)}) - \delta \norm{\alpha^{(j)}\bfs^{(j)}}^2 + \epsilon_j.
\end{align}
Furthermore, our convergence results assume the PR update of $\beta$ is given by \Cref{eq:PR}. The following Lemmas and ultimately convergence of \Cref{alg:vpal} rely on the following assumptions.
\begin{assumption}\label{assumption:1}
The lower level set,
    \begin{align}
        \calL_\epsilon(\bfx^0, \bfy^0) = \{(\bfx, \bfy)\;|\; h(\bfx, \bfy) \leq h(\bfx^0, \bfy^0) + \epsilon\} \subset{\bbR^{n_x + n_y}}
    \end{align}  is bounded, where $\epsilon$ is defined according to \Cref{equation:eps}.
\end{assumption}

\begin{assumption}\label{assumption:2}
    There exists a neighborhood $\calN$ of $\calL_\epsilon(\bfx^0, \bfy^0)$ for which the gradient of $\varphi$ with respect to $\bfx$, $\nabla_\bfx \varphi(\bfx, \bfy)$ is Lipschitz continuous.  That is, for some $L_\bfg > 0$, 
    \begin{align}
        \norm{\nabla_\bfx \varphi(\bfx_1, \bfy_1) - \nabla_\bfx \varphi(\bfx_2, \bfy_2)} = \norm{\bfg(\bfx_1,\bfy_1) - \bfg(\bfx_2, \bfy_2)} \leq L_\bfg\norm{(
            \bfx_1 - \bfx_2, \bfy_1 - \bfy_2)}
    \end{align} for all $(
        \bfx_1, \bfy_1),(
        \bfx_2, \bfy_2) \in \calN$.
    Additionally, a closed form solution to $\calZ(\bfx)$ defined according to \Cref{eq:yupdate} is given and $\calZ(\bfx)$ is Lipschitz continuous on $\calX = \{\bfx : (\bfx, \bfy) \in \calN\}$.  That is for some $L_\calZ > 0$, 
    \begin{align}
        \norm{\calZ(\bfx_1) - \calZ(\bfx_2)} \leq L_\calZ \norm{\bfx_1 - \bfx_2},
    \end{align}
    for all $\bfx_1, \bfx_2 \in \calX$.
\end{assumption}
\Cref{assumption:2} provides a bound $M > 0$ for which $\norm{\nabla_\bfx \varphi(\bfx, \bfy)} \leq M$ for all $(\bfx, \bfy)\in \calN$.  Together, \Cref{assumption:1} and \Cref{assumption:2} enable us to bound changes in function and gradient values following updates of both $\bfx$ and $\bfy$ which is essential for showing convergence to a stationary point.  This is particularly important in the proof of the following Lemmas.

\begin{lemma}
\label{lemma:levelset}
The iterates defined by \Cref{alg:vpal} remain in $\calL_{\epsilon}(\bfx^0, \bfy^0)$.  That is, 

\begin{align}
    (\bfx^{(j)}, \bfy^{(j)}) \in \calL_{\epsilon}(\bfx^0, \bfy^0)
\end{align}
for all $j \geq 0$.
\end{lemma}

\begin{proof}
    The optimality of each $\bfy^{(j)}$ and step size criteria for the update of $\bfx^{(j)}$ ensures that 
\begin{align}
    h(\bfx^{(j+1)}, \bfy^{(j+1)}) \leq h(\bfx^{(j+1)}, \bfy^{(j)}) \leq  \varphi({\bfx^{(j)}, \bfy^{(j)}}) - \delta \norm{\alpha^{(j)} \bfs^{(j)}}^2 + \epsilon_j +  r(\bfy^{(j)}) \leq h(\bfx^{(j)}, \bfy^{(j)}) + \epsilon,
\end{align} and thus inductively, the update of each $(\bfx^{(j)}, \bfy^{(j)})$  to $(\bfx^{(j+1)}, \bfy^{(j+1)})$ remain in $\calL_\epsilon(\bfx^0, \bfy^0)$. 
\end{proof}

\begin{lemma}\label{lemma:convlema}
   If \Cref{assumption:1} and \Cref{assumption:2} hold, then
\begin{align}
    \lim_{j \rightarrow \infty} \alpha^{(j)} \bfs^{(j)} = {\bf0} \textnormal{\quad and \quad} \lim_{j \rightarrow \infty} (\bfy^{(j+1)} - \bfy^{(j)}) = {\bf0}.
\end{align}
\end{lemma}

\begin{proof}

The statement that $\lim_{j \rightarrow \infty} \alpha^{(j)} \bfs^{(j)} = {\bf0}$ follows directly from \Cref{equation:eps}  and \Cref{equation:alpha_comp} as noted in \cite{zhou2013short}.   Furthermore, 
\begin{align} 
\lim_{j\rightarrow \infty} \norm{\bfy^{(j+1)} - \bfy^{(j)}} = \lim_{j\rightarrow \infty} \norm{\calZ(\bfx^{(j+1)}) - \calZ(\bfx^{(j)})} \leq \lim_{j\rightarrow \infty }L_{\calZ}\norm{\bfs^{(j)}} = 0. 
\end{align}
by \Cref{lemma:levelset} and \Cref{assumption:2}.
\end{proof}
\begin{lemma}\label{lemma:bndlma}
    
Let \Cref{assumption:1} and \Cref{assumption:2} hold.  If there is a constant $\tau > 0$ with
\begin{align}
    \norm{\bfg^j(\bfy^{(j)})} \geq \tau
\end{align} for all $j \geq 0$, then there exists a constant $R > 0$ such that 
\begin{align}
    \norm{\bfs^{(j)}} \leq R.
\end{align} 
\end{lemma}

\begin{proof}
By \Cref{assumption:2},

\begin{align}\label{eq:betabound}
    |\beta^{PR}_{j+1}| \leq \frac{\norm{\bfg^{j+1}(\bfy^{(j+1)})}\norm{\bfg^{j+1}(\bfy^{(j+1)}) - \bfg^j({\bfy^{(j)}})}}{\norm{{\bfg^j({\bfy^{(j)}})}}^2} \leq \frac{ML_\bfg(\norm{\bfs^{(j)}} + \norm{\bfy^{(j+1)}- \bfy^{(j)}})}{\tau^2} \rightarrow 0,
\end{align} as $j \rightarrow \infty$. From here, it follows as given in \cite[Proof of Lemma 2]{zhou2013short}, and so the remainder of the proof is omitted. 
\end{proof}

\Cref{lemma:bndlma} is crucial in the convergence proof given in \Cref{thm:convergence} since it ensures that $\norm{\bfs^{(j)}}$ and $\beta^{(j)}$ are bounded if $\bfg^j$ is bounded away from $0$.  \Cref{thm:convergence} is proved by contradiction, using this implication. 

\begin{theorem}\label{thm:convergence}
    Let \Cref{assumption:1} and \Cref{assumption:2} hold, and let $\{\bfx^{(j)}, \bfy^{(j)}\}$ be the sequence generated by \Cref{alg:vpal} with step size selection satisfying \Cref{equation:alpha_comp}, and $\beta^{(j)}$ updates computed according to \Cref{eq:PR}.  Then $\{\bfx^{(j)}, \bfy^{(j)}\}$  has a subsequence converging to a stationary point.

\end{theorem}

\begin{proof}
    We first note that selection of each $\bfy^{(j)}$ is optimal, and thus \begin{align}
       \textbf{0} \in  \nabla_\bfy \varphi(\bfx^{(j)}, \bfy^{(j)}) +  \partial_\bfy r(\bfy^{(j)}),
    \end{align}
    for all $j$.  Hence it is sufficient to show 
    \begin{align}\label{eq:xcondition}
        \liminf_{j \rightarrow \infty} \norm{\nabla_{\bfx}\varphi(\bfx^{(j)}, \bfy^{(j)})} = \liminf_{j \rightarrow \infty}\norm{\bfg^{j}(\bfy^{(j)})} = 0.
    \end{align} 
     Following \cite{zhou2013short}, we reduce to two cases and proceed by contradiction.  We first note that by \Cref{lemma:convlema} we have $\lim_{j \rightarrow \infty} \norm{\alpha^{(j)} \bfs^{(j)}} = 0$ which implies either $\liminf_{j \rightarrow \infty} \norm{\bfs^{(j)}} = 0$ or  $\liminf_{j \rightarrow \infty} \norm{\bfs^{(j)}} > 0$  and $\lim_{j \rightarrow \infty} \alpha^{(j)} = 0$.

    \begin{enumerate}
        \item In the first case, \begin{align}\label{eq:case1}
            \liminf_{j \rightarrow \infty} \norm{\bfg^{j}(\bfy^{(j)})} = \liminf_{j \rightarrow \infty} \norm{\bfs^{(j)} - \beta^{(j)} \bfs^{(j-1)}} \leq \liminf_{j \rightarrow \infty} \norm{\bfs^{(j)}} + \liminf_{j \rightarrow \infty}\norm{\beta^{(j)} \bfs^{(j-1)}}
        \end{align}
        By way of contradiction, we assume that $\liminf_{j \rightarrow_\infty} \norm{\bfg^{j}(\bfy^{(j)})} \neq 0$ which implies that there exists $\tau > 0$ for which $\norm{\bfg^j(\bfy^{(j)})} \geq \tau $ for all $j$ and thus \Cref{lemma:bndlma} holds. By \Cref{eq:betabound} of \Cref{lemma:bndlma}, $\beta^{(j)} \not\rightarrow \infty$, and thus \Cref{eq:case1} goes to $0$, which contradicts our assumption.

        \item In the second case, we have that $ \lim_{j \rightarrow \infty} \alpha^{(j)} = 0$, and by definition of $\alpha^{(j)}$, we have that $\tilde\alpha^{(j)} = \alpha^{(j)}/\rho$ does not satisfy our step size criteria.  In particular, for all $j \geq 0$, we have that

        \begin{align}
            \varphi(\bfx^{(j)} + \tilde\alpha^{(j)}\bfs^{(j)}, \bfy^{(j)}) >   \varphi(\bfx^{(j)}, \bfy^{(j)}) - \delta \norm{\tilde\alpha^{(j)}\bfs^{(j)}}^2,
        \end{align}

        which gives 
        \begin{align}
            \frac{\varphi(\bfx^{(j)} + \tilde\alpha^{(j)}\bfs^{(j)}, \bfy^{(j)}) - \varphi(\bfx^{(j)}, \bfy^{(j)})}{\tilde\alpha^{(j)}} \geq -\delta \tilde\alpha^{(j)} \norm{\bfs^{(j)}}^2.
        \end{align}

        Since this is true for any $j$, the remainder of the proof follows as in \cite[Proof of Theorem 1]{zhou2013short} where the mean-value theorem, and \Cref{lemma:bndlma} are used to derive a contradiction.
    \end{enumerate}
\end{proof}

\Cref{thm:convergence} guarantees convergence (in the $\liminf$ sense) of our algorithm with relatively loose assumptions on the objective function.  With some restrictions on the setup of the EEG source localization problem, we can show that \Cref{equation:objective} satisfies \Cref{assumption:1} and \Cref{assumption:2} therefore establishing convergence for our specific problem.

\begin{theorem}
    If $\bfA = \begin{bmatrix}
    \bfI_T \otimes \bfL \\ \lambda(\bfD_1^\top \otimes \bfI_n)
\end{bmatrix}$ has full column rank, then $f_{\rm joint}(\bfx, \bfy)$ as defined in \Cref{eq:fjoint} is strictly convex and satisfies \Cref{assumption:1} and \Cref{assumption:2} for any initial point, $(\bfx^0, \bfy^0)$.
\end{theorem}

\begin{proof}
    Writing \Cref{eq:fjoint} as 
    \begin{align}
        f_{\rm joint}(\bfx, \bfy) = \norm[2]{\begin{bmatrix}
            \bfA & \bf0 \\
            \tilde{\bfD}_2 & -\lambda\bfI_{mT}
        \end{bmatrix}\begin{bmatrix}\bfx \\ \bfy\end{bmatrix} - \begin{bmatrix}
            \bfb \\ -\eta\bfc
        \end{bmatrix}}^2 + \mu \norm[1]{\bfy}
    \end{align}
we have that $\bfA$ full column rank implies that $\begin{bmatrix}
            \bfA & {\bf0} \\
            \tilde{\bfD}_2 & -\lambda\bfI_{mT}
\end{bmatrix}$ is full column rank. Hence, $f_{\rm joint}(\cdot, \cdot)$ is strictly convex, and coercive, therefore $\calL_\epsilon(\bfx^0, \bfy^0)$ is compact for any $(\bfx^0, \bfy^0)$ satisfying \Cref{assumption:1}].  The gradient of $f_{\rm joint}(\bfx, \bfy)$ with respect to $\bfx$ is affine, and thus Lipschitz on the compact lower level set.  The continuity of $\bfZ(\cdot)$ and compactness of the lower level set implies that $\bfZ(\cdot)$ is also Lipschitz continuous and \Cref{assumption:2} holds. 
\end{proof}

\section{Numerical Experiments}\label{section:numExp}

In this section, we investigate EEG experiments conducted on simulated data.  We first introduce our sequential \emph{windowed} variation of {\tt VPAL} in \Cref{section:window}. A simulation study is introduced in \Cref{sub:data} followed by a comparison study in \Cref{section:comparison_study}. We evaluate the scalability of our algorithms to larger datasets in \Cref{section:scalability}. A discussion about hyperparameter selection is included in \Cref{section:paramselection}.

\subsection{A Windowed Approach}\label{section:window}
We will illustrate the performance of {\tt VPAL} compared to {\tt ADMM} and {\tt FISTA} in \Cref{section:comparison_study}.  However, {\tt ADMM}, {\tt FISTA}, and {\tt VPAL} compute the entire dynamics simultaneously, limiting real-time potential. 
 We now discuss how to use a \emph{moving window} approach based on {\tt VPAL} to instead compute later time points sequentially, where a \emph{window} refers to a collection of time points in sequence, i.e. $W_j = \{t_j, t_{j+1}, \dots, t_{j+w}\}$. Heuristically, this approach first computes the reconstruction for a window and then uses this reconstruction as an initial guess to compute the reconstruction for the next window.  If the windows are small enough, the previous reconstruction should be a good initialization and facilitate fast accurate convergence.

We can frame this mathematically as follows: assume we are given a solution $\bfX^{(W_{1})} = [\bfx_{1}, \bfx_{2}, \dots, \bfx_{w+1}]$ using data $\bfB^{(W_{1})} = [\bfb_1,\bfb_2, \ldots,\bfb_{w+1}]$. Moving to later windows of time, $W^{2}, W^{3}, \dots, W^{J}$, the reconstruction generated from the previous window can be used to initialize the estimation of the current time window.  Implementation of this algorithm, denoted as {\tt VPAL$_{\tt W}$},   is summarized in \Cref{alg:vpalWindowed}.

\begin{algorithm}[hbt!]
\caption{{\tt VPAL$_{\tt W}$}}\label{alg:vpalWindowed}
    \begin{algorithmic}[1]
    \State \textbf{input} $\bfL$, $\bfB$, $\{W_j\}_{j = 1}^J$

    \State $\bfX^{(W_1)}$ = {\tt VPAL}($\bfL$, $\bfB^{(W_1)}$)\label{ln:initalvpalw}

    \For{$j = 2$ to $J$}

        \State $\bfX^{(W_j)} = $ {\tt VPAL}($\bfL$, $\bfB^{(W_j)})$ with initial guess $\bfX_{0} = \bfX^{(W_{j-1})}$\label{ln:innerloopvpalw}
    \EndFor
    \State \textbf{output} $\bfX$ 
    \end{algorithmic}
\end{algorithm}

In our given implementation, the windows need not be disjoint. To ensure smooth transition between reconstructions at adjacent windows, and thus maintain time dependencies,  we opt to have the last time point in the previous window overlap with the first time point in the current window.  Furthermore, the windows need not be the same length, so long as the initial guess is sliced so it is an appropriate size to initialize the next window.  For simplicity, we use a configuration where all windows are of the same length and we reconstruct one new time point in each window--that is 
\begin{align}
    W^k = \{t_j, t_{j + 1}, \dots t_{j + w}\} \textnormal{ and } W^{k+1} = \{t_{j+w}, t_{j + w+ 1}, \dots t_{j + 2w}\},
\end{align}
 where we set $w = 1$.

The hyperparameters used for {\tt VPAL} calls in \Cref{ln:initalvpalw} and \Cref{ln:innerloopvpalw} can vary.  In our study, \Cref{ln:innerloopvpalw} was run for significantly fewer iterations than \Cref{ln:initalvpalw}.  Due to the initialization using the previous window, this was found to be sufficient for quick high-fidelity reconstructions at each additional window.  It is important to note that a similar moving window approach could be implemented analogously using {\tt ADMM} or {\tt FISTA} for reconstruction of time points.  However as will be made clear in \Cref{section:comparison_study}, {\tt VPAL} is significantly faster and thus the only algorithm suitable for real-time reconstructions.

\subsection{EEG Dataset}\label{sub:data}

To validate our methods, we set up simulation studies leveraging the ICBM-NY dataset \cite{HAUFE2008726, NYHEAD} which includes highly detailed models of the cortical mesh with different resolutions.  We assess the performance of our novel {\tt VPAL} and {\tt VPAL$_{\tt W}$} algorithms on three meshes with varying sizes $n =  2,\!004, 10,\!016$, and $74,\!382$ nodes and only $p = 231$ electrodes. We compare these algorithms to {\tt ADMM} and {\tt FISTA} on a small- and medium-sized mesh to establish a baseline. To showcase its ability to handle large-scale datasets and investigate their scalability, we apply our methods to a large mesh, where traditional algorithms like {\tt ADMM} and {\tt FISTA} become computationally prohibitive.

To establish a controlled experimental environment, we first generated a simplified, small dataset characterized by a single active brain region.  Here, a random selection of nodes is activated and their influence propagates to neighbors smoothly over time. We represent the voltage on the mesh by the matrix $\bfX_{\text{true}}$.  Simulated observations $\bfB$ are generated according to the model $\bfB = \bfL\bfX_{\text{true}}+ \bfE$, where $\bfE$ is an additive 10\% Gaussian white noise realization to mimic measurement noise. 
Concerning the small dataset, the comparison dataset is generated by simulating a single-source propagation on the mesh for 20 time steps, resulting in $\bfX_{\textnormal{true}}^{\textnormal{C}} \in \bbR^{2,004 \times 20}$.

To compare how our methods perform as the size of the data and model increases, we generated additional datasets,  $\bfX_{\textnormal{true}}^{\textnormal{S}} \in \bbR^{2,004 \times 50}$, $\bfX_{\textnormal{true}}^{\textnormal{M}} \in \bbR^{10,016 \times 50}$ and $\bfX_{\textnormal{true}}^{\textnormal{L}} \in \bbR^{74,382 \times 50}$.  As algorithmic runtime varies across datasets, we generated three different versions of each dataset using different random seeds.  Runtime results and performance metrics were averaged across all three datasets.

An additional 10 test sets were created to evaluate the scalability of the proposed methods with respect to the number of time points, $T$.  These datasets, $\bfX^{\textnormal{S}_T}_{\textnormal{true}} \in \bbR^{2,004 \times T}$ for $T = 10, 20, \dots, 100$ were generated via the same procedure before, and contain multiple activation sites.  Again, due to variance in runtime, we generated three unique datasets for each size and reported the average runtime. 

All experiments were conducted on a 2019 6-Core Macbook Pro with 16 GB of DDR4 RAM. \emph{To ensure reproducibility, the elementary code will be released upon acceptance.}

\subsection{Comparison Study}\label{section:comparison_study}

To analyze the properties of the algorithms we introduced, we compare {\tt VPAL} and {\tt VPAL$_{\tt W}$} to the standard methods for solving elastic net problems, i.e., {\tt ADMM} and {\tt FISTA} as well as {\tt sLORETA}.   To further ensure a fair comparison, each method's hyperparameter was selected through a comprehensive hyperparameter study. Details of the hyperparameter study are outlined in \Cref{section:paramselection}. Once the best parameters in the given range were identified the respective algorithm was run until standard termination criteria with tolerance of $10^{-5}$ were fulfilled \cite{doi:10.1137/1.9781611975604} or 1,000 iterations were reached. Additionally, we compare to {\tt sLORETA} with an approximately optimal regularization parameter with respect to relative error. 

To test \Cref{alg:vpalWindowed}, we consider the same dataset used in \Cref{section:comparison_study}, and perform all {\tt VPAL} runs using the same setup, except only allowing {\tt VPAL} calls in \Cref{ln:innerloopvpalw} run for $100$ iterations.  Empirically, there were enough iterations for the algorithm to produce accurate reconstructions while maintaining quick reconstructions for each time point.  The results are summarized in \Cref{fig:compresults}, \Cref{fig:compplots}, and \Cref{table:cres}.  Here PSNR is a common reconstruction measure and refers to the peak signal-to-noise ratio \cite{hore2010image}. We also record the reconstruction error, or relative error (RE) of each algorithm.  These are defined as

\begin{align}
    \textnormal{PSNR} = 10 \log_{10}\left(\frac{\max(\bfX)^2}{\norm[\fro]{\bfX-\bfX_{\textnormal{true}}}^2}\right)  \quad \textnormal{ and } \quad \textnormal{RE} = \frac{\norm[\fro]{\bfX-\bfX_{\textnormal{true}}}}{\norm[\fro]{\bfX_{\textnormal{true}}}}.
\end{align}

To gain insight into the structure of the reconstruction, we consider two additional metrics: source error distance (SED) and structural similarity index (SSIM).  SED measures the distance between the active source and the predicted active source in the reconstruction averaged across time points.  SSIM heuristically measures the human-perceived quality of an image compared to a reference image--see \cite{1284395} for further details.  Since $\ell_1$-regularization was implemented to impose sparsity on the finite difference of mesh values, we also report a sparsity measure defined as the number of non-zero elements in $\bfD_2\bfX$ divided by the number of elements of $\bfX$.  For simplicity, we refer to this as the sparsity of $\bfX$ and it should not be confused with the number of non-zero elements of $\bfX$.

\begin{figure}[hbt!]
    \centering
    \includegraphics[scale=0.3]{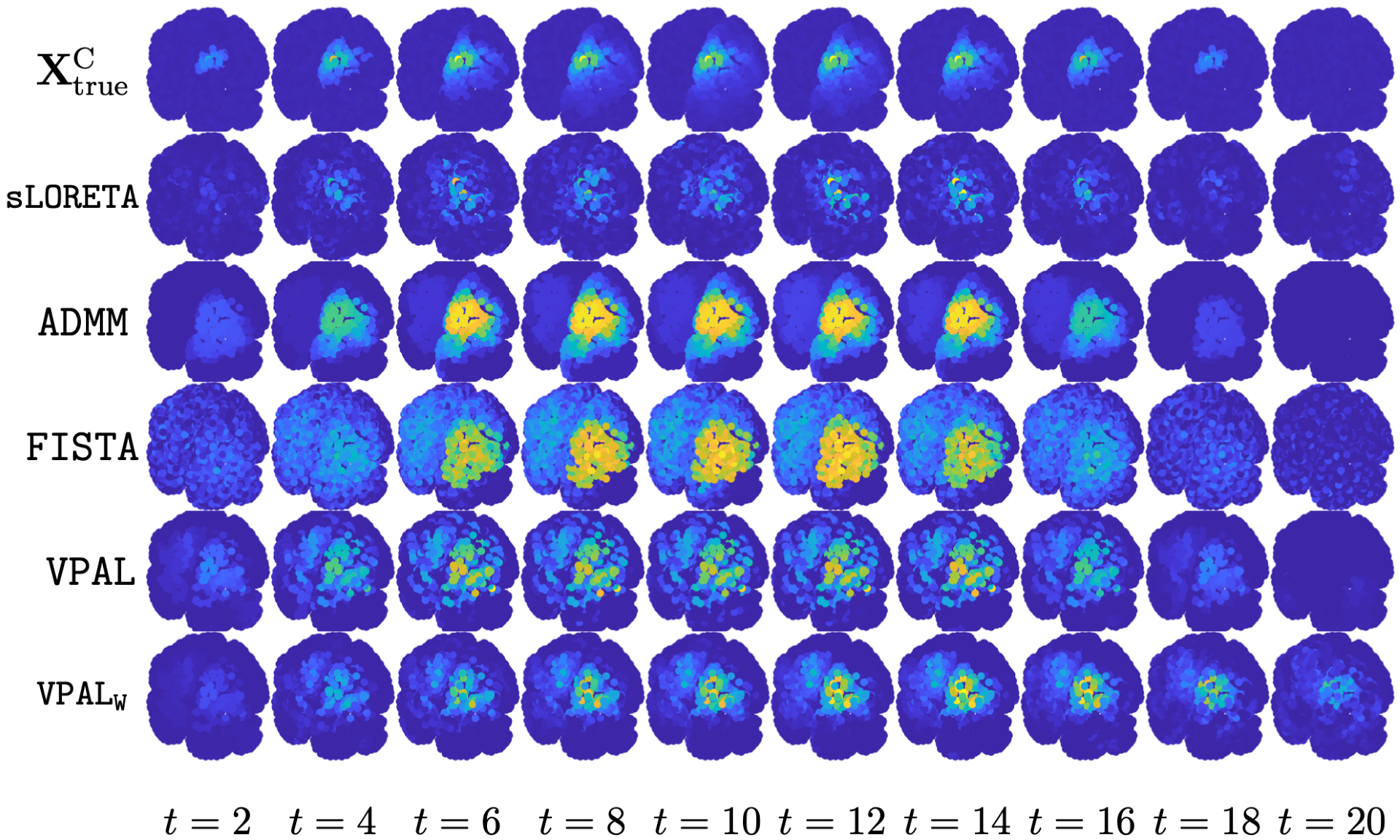}
    \caption{Visualization of {\tt sLORETA}, {\tt ADMM}, {\tt FISTA}, {\tt VPAL} and {\tt VPAL$_{\tt W}$}  reconstructions, compared against ground truth comparison dataset, $\bfX_{\textnormal{true}}^{\textnormal{C}}$. An approximately optimal regularization parameter, $\lambda = 0.0059$, was used for the {\tt sLORETA} reconstruction. }

    \label{fig:compresults}
\end{figure}

\begin{table}[hbt!]
    
    \centering
    
        \begin{tabular}{c||c|c|c|c|c|c}
         \textbf{Algorithm} & \textbf{Runtime (s)} & \textbf{PSNR} & \textbf{RE} & \textbf{SSIM} & \textbf{SED} & \textbf{Sparsity} \\
         \hline
         \hline
         {\tt sLORETA} &\textbf{0.33} & 22.53 & 1.046 & 0.298 & 28.44 & 0.437 \\
         \hline
         {\tt ADMM} & $333$  & \textbf{27.01} & \textbf{0.679} & \textbf{0.788} &  30.43 & \textbf{0.068}\\
         \hline
         {\tt FISTA} & $1,\!318$ & 24.67 &  0.892 &  0.618 & 34.05 & 0.491\\
         \hline
          {\tt VPAL} & 36 &  24.76 & 0.884 & 0.685 & 46.96 & 0.184\\
         \hline
         \small{\tt VPAL$_{\tt W}$} & 19 &  25.54 & 0.787 & 0.493  & \textbf{17.86} & 0.371\\
         \end{tabular} 
         \caption{Runtime and accuracy metrics of {\tt ADMM} {\tt FISTA} and {\tt sLORETA} against proposed {\tt VPAL} and {\tt VPAL$_{\tt W}$}.  The average runtime across 3 trials is presented in the table. Note, the sparsity of $\bfD_2\bfX_{\rm true}^{\rm C}$ is 0.452}.         \label{table:cres}
 \end{table}

\begin{figure}[htb!]
    \centering
    \includegraphics[width=0.4\linewidth]{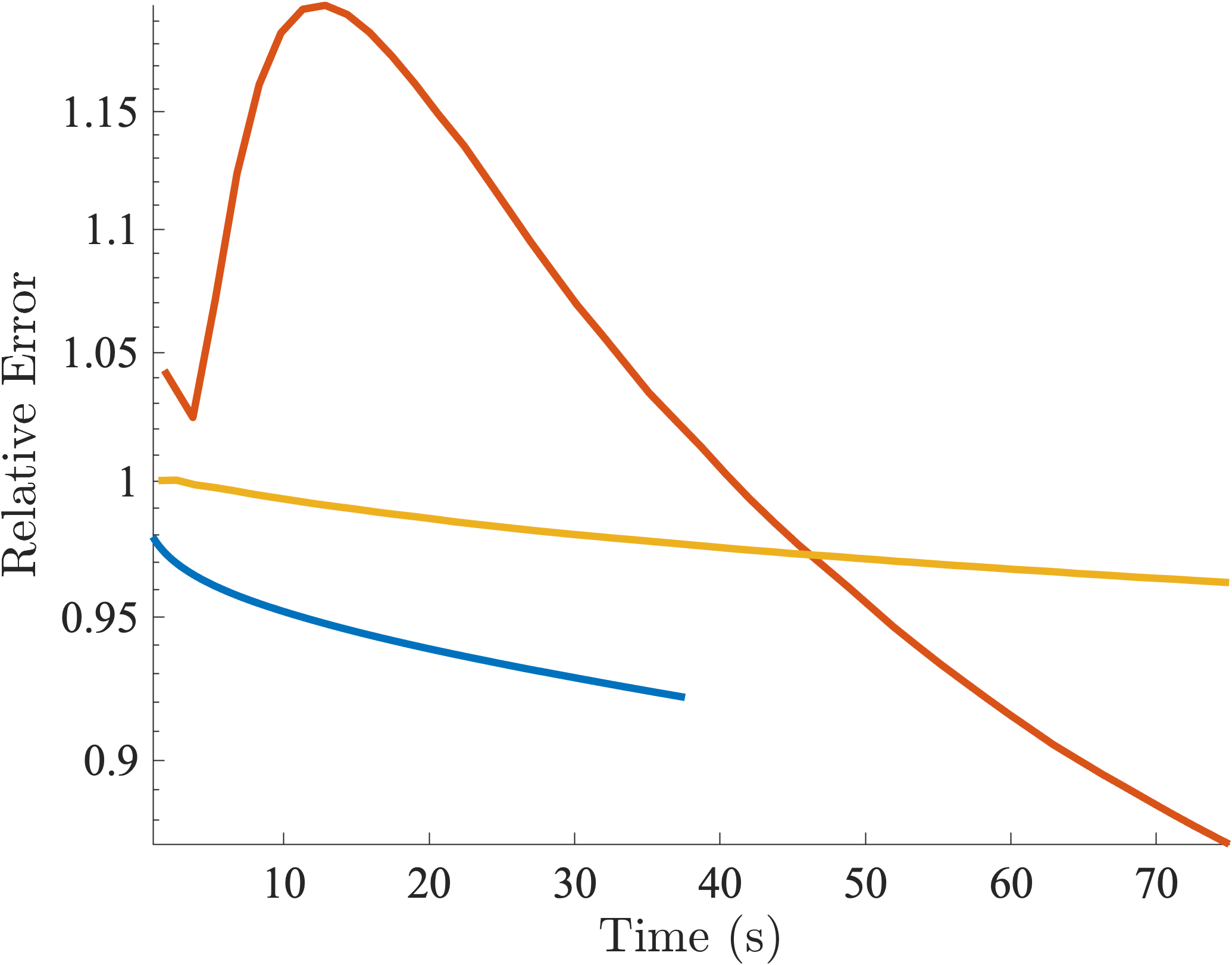}
        \includegraphics[width=0.4\linewidth]{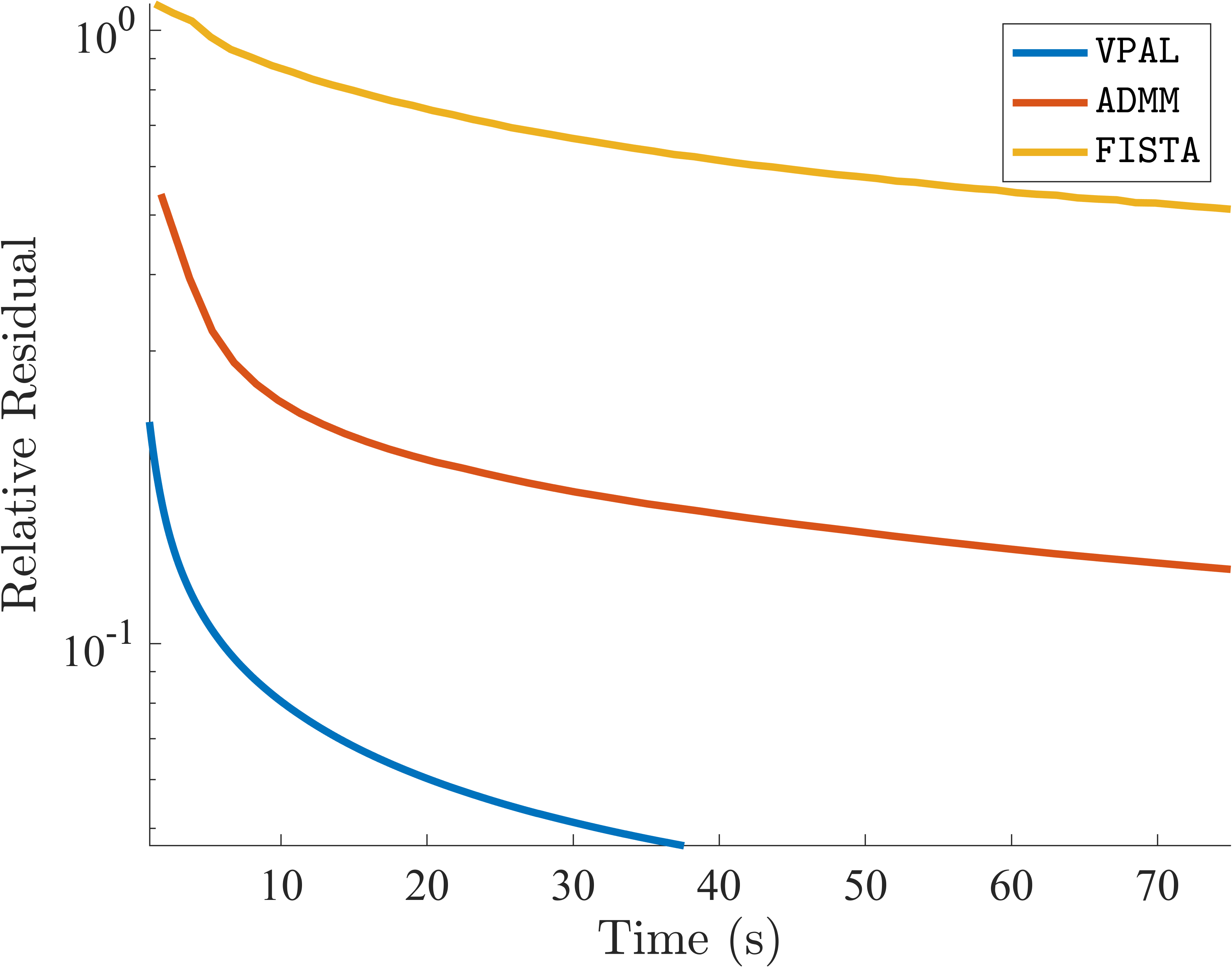}
    \caption{Relative error and residual plot against time for {\tt VPAL}, {\tt ADMM} and {\tt FISTA} (including the computations of residuals and errors in contrast to values reported in \Cref{table:cres}). }
    \label{fig:compplots}
\end{figure}

\subsection{Scalability Comparison}\label{section:scalability}
 
 To analyze the scalability of our algorithms to larger datasets with higher-resolutions and meshes, we compare the performance of {\tt VPAL} and {\tt VPAL$_{\tt W}$} to {\tt ADMM} on datasets, $\bfX_{\textnormal{true}}^{\textnormal{S}}$, $\bfX_{\textnormal{true}}^{\textnormal{M}}$.  Additionally, we present the results of {\tt VPAL} and {\tt VPAL$_{\tt W}$} on $\bfX^{\textnormal{L}}_\textnormal{true}$. An {\tt ADMM} reconstruction could not be computed for $\bfX^{\textnormal{L}}_\textnormal{true}$ since required computations became intractable. Note, we omit comparisons to {\tt FISTA} due to their inferior performance presented in \Cref{section:comparison_study} and {\tt sLORETA}, since the focus of this section is the relative performance of $\ell_1$-regularized methods.  Additionally, we only consider the PSNR and RE for accuracy since the primary focus of this section is on the runtime of our methods on large datasets. We performed reconstruction experiments on each dataset and compared the fidelity as well as the algorithm runtime after $500$ iterations or convergence.

\begin{figure}[hbt!]
    \centering
    \includegraphics[scale = 0.3]{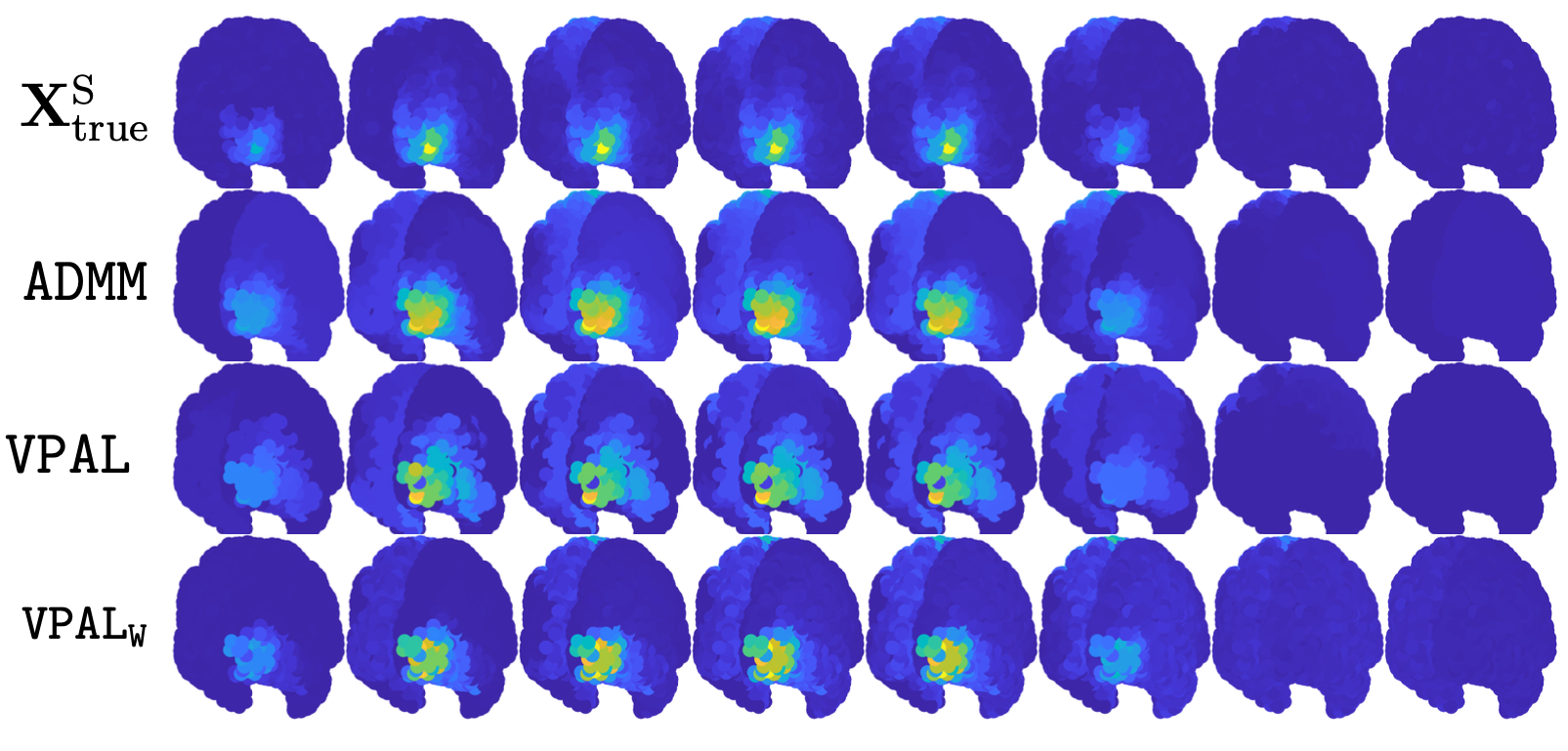}
    \quad

    \includegraphics[scale = 0.3]{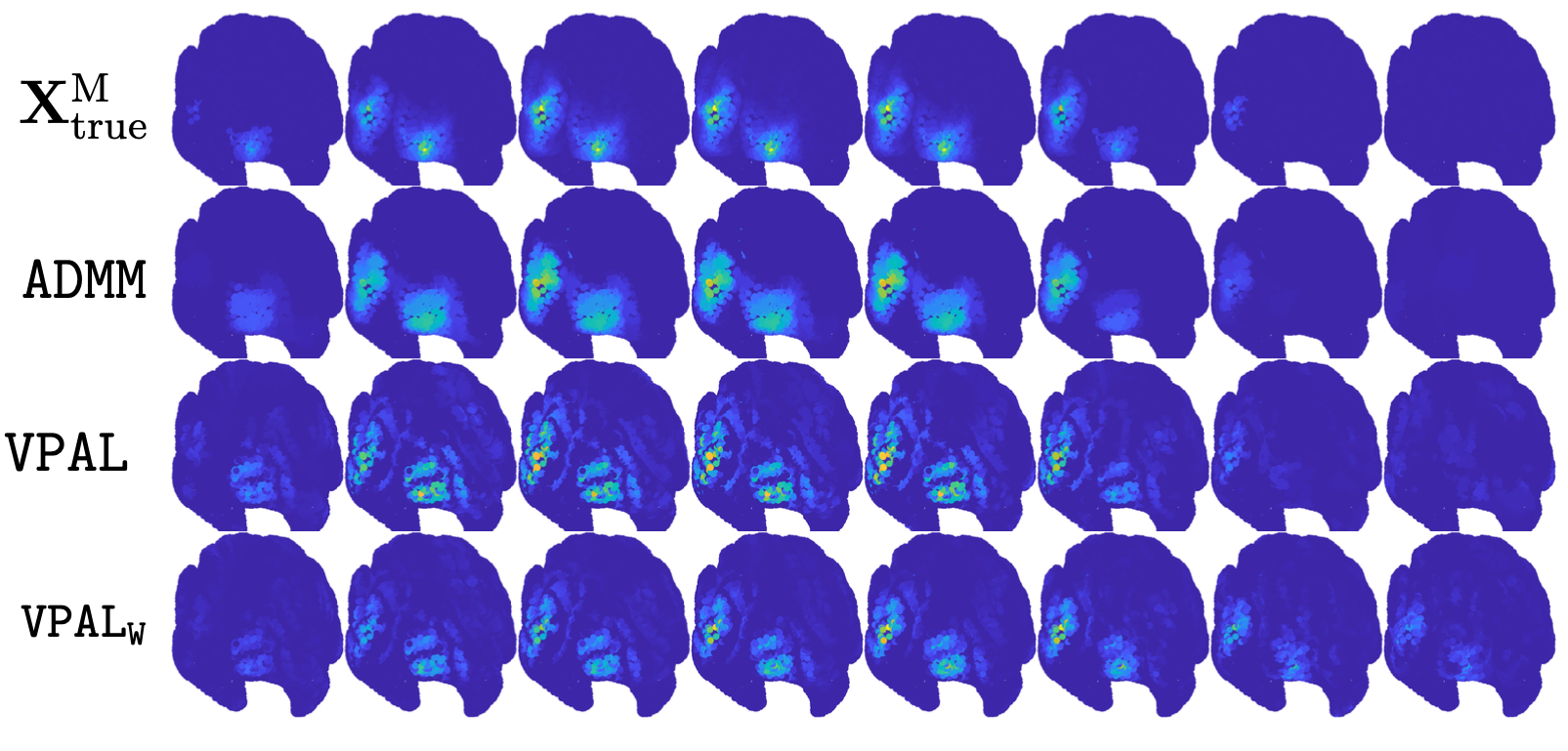}
    \quad

    \includegraphics[scale = 0.3]{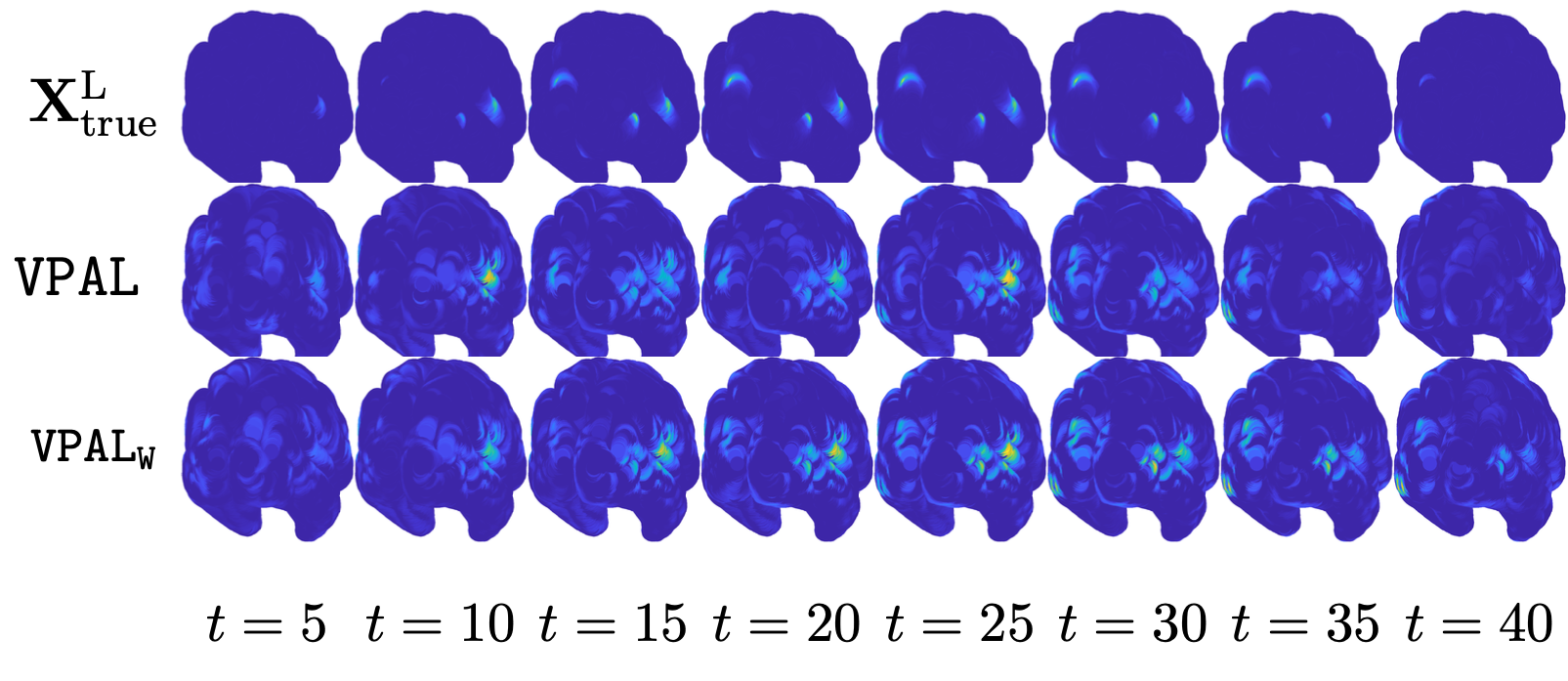}
    \caption{Visualization of {\tt ADMM}, {\tt FISTA}, {\tt VPAL} and {\tt VPAL$_{\tt W}$} reconstructions, compared against ground truth scalability datasets, $\bfX_{\textnormal{true}}^{\textnormal{S}}$, $\bfX_{\textnormal{true}}^{\textnormal{M}}$, $\bfX_{\textnormal{true}}^{\textnormal{L}}$.  Visualization is shown for the first of the three trial datasets.  Since most interesting dynamics occur before $t = 40$, later time points are not shown.}  
    \label{fig:compresults_scala}
\end{figure}

\begin{figure}[hbt!]
    \centering
    \includegraphics[width=0.5\linewidth]{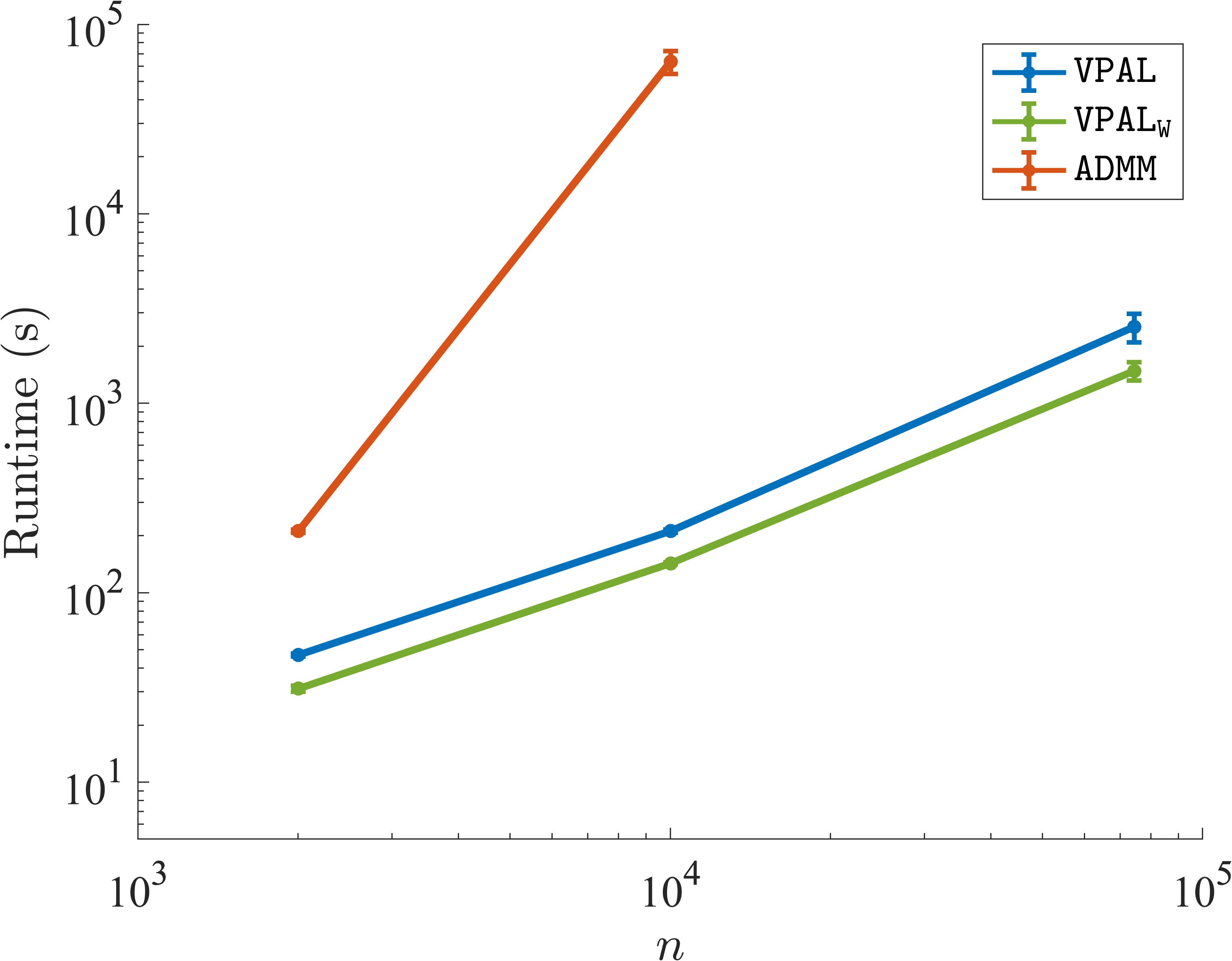} 
    \caption{ Average runtime comparison of {\tt ADMM}, {\tt VPAL} and {\tt VPAL$_{\tt W}$} as $n$ is increased.  Error bars correspond to the standard deviation of runtime across three unique datasets.  }
    \label{fig:scalaplot}   
\end{figure}

\begin{table}[H]
\centering
\footnotesize\begin{tabular}{c||c|c|c|c|c|c} 
  Method &  PSNR $\bfX^{\textnormal{S}}_{\textnormal{true}}$  &  RE $\bfX^{\textnormal{S}}_{\textnormal{true}}$ & PSNR $\bfX^{\textnormal{M}}_{\textnormal{true}}$ &  RE $\bfX^{\textnormal{M}}_{\textnormal{true}}$& PSNR  $\bfX^{\textnormal{L}}_{\textnormal{true}}$ & RE  $\bfX^{\textnormal{L}}_{\textnormal{true}}$\\ \hline\hline
  {\tt ADMM} & \textbf{19.772} \footnotesize{$\pm$0.545} & \textbf{0.681} \footnotesize{$\pm$0.025} & 23.442 \footnotesize{$\pm$2.330} & 1.170 \footnotesize{$\pm$0.447} & n/a & n/a\\
  {\tt VPAL} &  16.968 \footnotesize{$\pm$0.851}  & 0.940 \footnotesize{$\pm$0.004} &  24.784 \footnotesize{$\pm$1.90}  & \textbf{0.946} \footnotesize{$\pm$0.020} &\textbf{24.495} \footnotesize{$\pm$0.565}  &{\bf 0.987} \footnotesize{$\pm$0.005}\\ 
  \footnotesize{{\tt VPAL$_{\tt W}$}} &  18.325 \footnotesize{$\pm$0.682} & 0.804 \footnotesize{$\pm$0.032}&\textbf{25.165} \footnotesize{$\pm$1.686} &  0.9062 \footnotesize{$\pm$0.046}&  24.442 \footnotesize{$\pm$0.514} &  0.985 \footnotesize{$\pm$0.004}\\
\end{tabular}
    \caption{PSNR comparison of {\tt ADMM}, {\tt VPAL} and {\tt VPAL$_{\tt W}$} as $n$ is increased.  Metrics are averaged across all three trial datasets and standard deviation is reported in parenthesis.} 
    \label{tab:my_label}
\end{table}

Following the same procedure, we measured the runtime of {\tt VPAL} and {\tt VPAL$_{\tt W}$} on a series of datasets where $T$ is varied from $10$ to $100$ in increments of $10$.  The results of this experiment can be seen in \Cref{fig:time_scaling}.  Results in \Cref{fig:scalaplot} are averaged over 3 runs, with different datasets for each.

\begin{figure}[hbt!]
    \centering \includegraphics[width=0.5\linewidth]{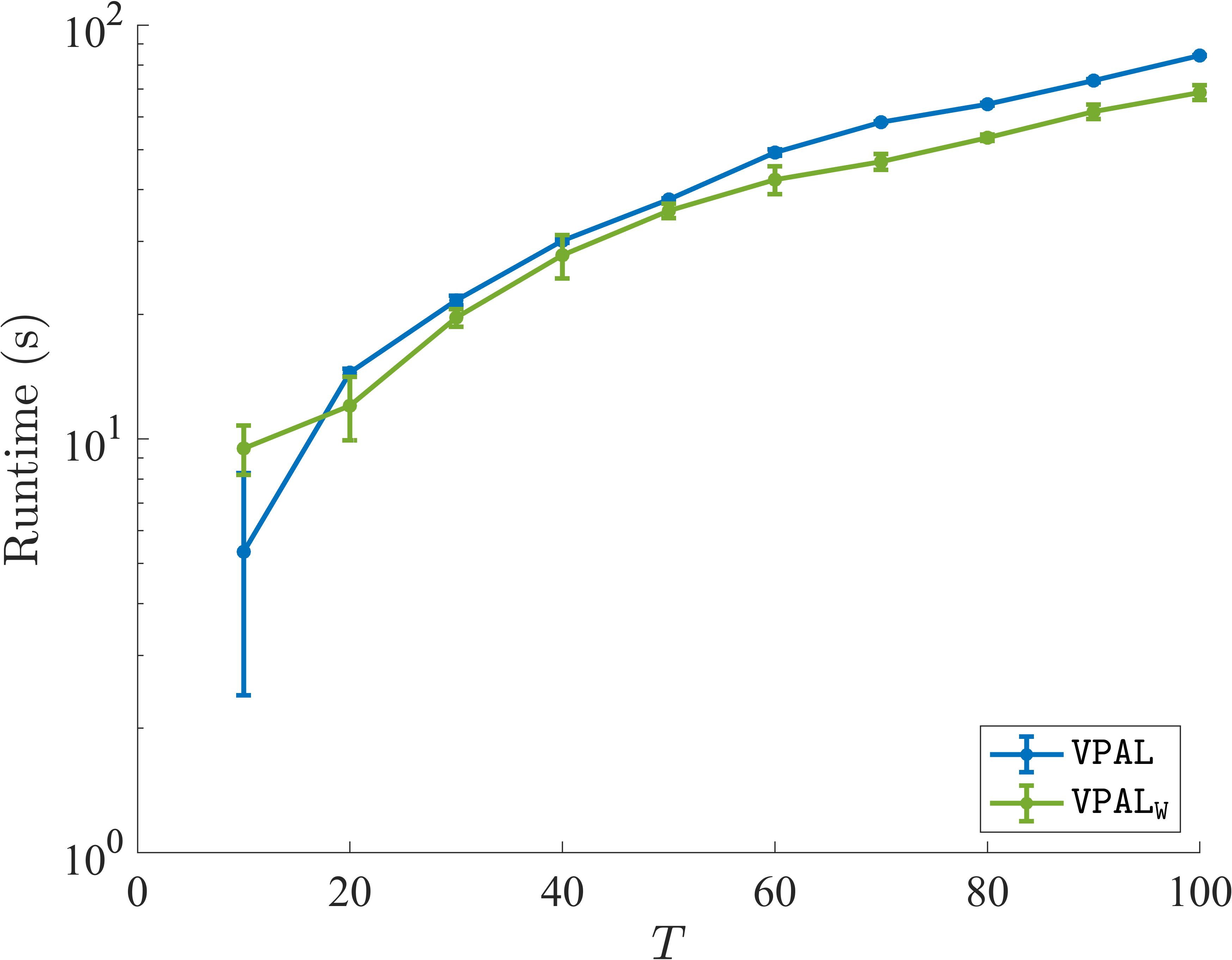}
    \caption{Average runtime comparision of {\tt VPAL} and {\tt VPAL$_{\tt W}$} for $T = 10, 20, \dots 100$.  Average runtime across three unique datasets is plotted, with error bars corresponding to the standard deviation.}
    \label{fig:time_scaling}
\end{figure}

\section{Discussion}\label{section:Discussion}

In \Cref{section:comparison_study}, we compared the performance of {\tt sLORETA}, {\tt ADMM}, {\tt VPAL}, {\tt FISTA} and {\tt VPAL$_{\tt W}$} on the \emph{small} source-localization task. Consistent with expectations for this highly underdetermined EEG source localization problem, the relative error for all algorithms is substantial, as depicted in \Cref{fig:compplots}. However, as illustrated in \Cref{fig:compresults}, all methods successfully localized the active region, exhibiting qualitatively comparable levels of performance. \Cref{table:cres} confirms that {\tt ADMM} achieved the best performance in terms of PSNR, RE, SSIM while all other methods showed similar performance in these metrics.  All $\ell_1$-regularized methods resulted in reconstructions with sparser 
$\bfD_2\bfX$ than {\tt sLORETA}, except for {\tt FISTA}, but it is difficult to compare the sparsity of $\ell_1$-regularized against each other due to varied sparsity parameter, $\mu$.  Moreover, {\tt VPAL$_{\tt W}$} performed best with respect to SED, and all $\ell_1$-regularized methods significantly outperformed {\tt sLORETA} in accuracy metrics except for SED.  Nevertheless, regarding the runtime, {\tt VPAL} and {\tt VPAL$_{\tt W}$} outperform {\tt ADMM} and {\tt FISTA} on this small simulation study, but are outperformed by the instantaneous reconstruction of {\tt sLORETA}, as expected. This motivates us to investigate the scalability of the algorithms on larger datasets.

Improved runtime for {\tt VPAL} and {\tt VPAL$_{\tt W}$} is further apparent in the results of \Cref{section:scalability}.  For all resolutions of the mesh, $n$, tested, {\tt ADMM} is consistently slower than our proposed methods by several orders of magnitude. As depicted in \Cref{fig:scalaplot}, the runtime of {\tt ADMM} increases more rapidly with respect to mesh size $n$ compared to the other two algorithms. This leads to {\tt ADMM} becoming computationally intractable for high-resolution datasets. 
 Both {\tt VPAL} and {\tt VPAL$_{\tt W}$} scale more favorably than {\tt ADMM}, with {\tt VPAL}$_{W}$ exhibiting superior computational efficiency compared to {\tt VPAL} for larger mesh sizes.

We also note that while {\tt ADMM} outperforms {\tt VPAL$_{\tt W}$} for the small mesh, {\tt VPAL$_{\tt W}$} offers better performance metrics for the medium-sized mesh in PSNR.  Qualitatively, {\tt VPAL$_{\tt W}$} produces better reconstructions than {\tt VPAL}.  This is particularly evident in the largest dataset visualization where {\tt VPAL} captures the strongest signal and fails to identify some other signals (\Cref{fig:compresults_scala}). {\tt VPAL$_{\tt W}$}, on the other hand, identified all the signals present in the ground truth.

As $T$ is varied (see \Cref{fig:time_scaling}), we observe that {\tt VPAL$_{\tt W}$} has slower runtimes than {\tt VPAL} for small values of $T$.  However, the growth rate of the runtime of {\tt VPAL$_{\tt W}$} is slower on average than that of {\tt VPAL} and thus {\tt VPAL$_{\tt W}$} offers improved runtime for $T$ large.  {\tt VPAL$_{\tt W}$} offers the additional advantage of nearly real-time reconstruction of incoming data points.  For small meshes, the algorithm can reconstruct incoming time points in $0.538 \pm 0.023$ seconds on average.  For the medium and large datasets, the average reconstruction time per time point is scaled to $2.465 \pm 0.020$ and $25.667 \pm 2.607$ seconds, respectively.   

The efficient runtime of {\tt VPAL} and {\tt VPAL$_{\tt W}$} benefit significantly from our data-efficient implementation of the {\tt graphTV} operator as defined in \Cref{eq:graphTV}, as well as our linearized step size selection routine. The {\tt graphTV} operator can be efficiently implemented in MATLAB as follows:
\begin{align}
    \bfD_2\bfX  = \bfw \hadamard (\bfX(\bfn_1, :) - \bfX(\bfn_2, :)),
\end{align}
where $\bfw$ is the diagonal of $\bfW$. Here, $\bfn_1(i), \bfn_2(i)$ are the indices of the two nodes that define the $i$-th edge.  For the EEG application, we assume unweighted, undirected edges, and thus $\bfW = \bfI_{m}$.  We note that, due to implementing the mesh as a graph, this is an efficient computation, depending linearly on the number of edges $m$.

\section{Conclusions \& Outlook}\label{section:ConclusionsandOutlook}

In this work, we presented a novel approach to address the EEG source localization inverse problem and leverage  a data-efficient graph total variation ({\tt graphTV}) operator on the cortical mesh to regularize spatial aspects of the reconstruction coupled with regularization on the time dynamics.  Our main contribution was the introduction of two fast solvers for this inverse problem--the \emph{variable projected augmented Lagrangian nonlinear conjugate gradient} algorithm ({\tt VPAL}) and the windowed variation, {\tt VPAL$_{\tt W}$}.  Additionally, we verified the theoretical backing to our {\tt VPAL} algorithm, proving convergence for separable convex functions subject to linear constraints--in particular, the EEG source localization problem.

We investigated the performance of our proposed methods in terms of both runtime and reconstruction quality. Our novel algorithms demonstrated comparable performance to standard elastic net solvers like {\tt ADMM} and {\tt FISTA} while offering significantly faster runtimes and superior scalability. {\tt VPAL$_{\tt W}$} is particularly efficient in handling large datasets and is well-suited for real-time applications, such as monitoring brain activity in clinical settings.

While our study highlights the advantages of our novel approach, it also has several limitations. Notably, hyperparameter selection was performed using a brute-force approach, and further statistical and theoretical analysis is needed to optimize this process. Additionally, more rigorous termination criteria are required. Furthermore, our evaluation relied on synthetic datasets, and future work will involve testing our methods on real-world data to assess their clinical applicability. 

In future research, we aim to investigate applications of variable projection-based approaches in solving nonlinear inverse problems and further develop theoretical foundations. We also plan to explore the potential of {\tt VPAL$_{\tt W}$} as a filtering technique in data assimilation. The proposed {\tt graphTV} operator and algorithms have broad applicability to various data-intensive graph-based inverse problems, including those involving different graph structures (e.g., directed, weighted, etc.).

\vspace{0.5cm}

{\bf Acknowledgment.} 
This work is partially supported by the 2023--2024 Emory URC interdisciplinary award (Chung) and by the National Science Foundation (NSF) under grant  DMS-2152661 and [DMS-1913136, DMS-2152704] for Chung and Renaut, respectively. Any opinions, findings, conclusions, or recommendations expressed in this material are those of the authors and do not necessarily reflect the views of the National Science Foundation.

\printbibliography

@article{chung2022variable,
    title = {A variable projection method for large-scale inverse problems with $\ell_1$ regularization},
journal = {Applied Numerical Mathematics},
volume = {192},
pages = {297-318},
year = {2023},
issn = {0168-9274},
doi = {https://doi.org/10.1016/j.apnum.2023.06.015},
url = {https://www.sciencedirect.com/science/article/pii/S0168927423001800},
author = {Matthias Chung and Rosemary A. Renaut}
}

@article{eegreview,
author = {Beres, Anna},
year = {2017},
month = {12},
pages = {},
title = {Time is of the Essence: A Review of Electroencephalography (EEG) and Event-Related Brain Potentials (ERPs) in Language Research},
volume = {42},
journal = {Applied Psychophysiology and Biofeedback},
doi = {10.1007/s10484-017-9371-3}
}

@article{zhou2013short,
  title={A short note on the global convergence of the unmodified PRP method},
  author={Zhou, Weijun},
  journal={Optimization Letters},
  volume={7},
  number={6},
  pages={1367--1372},
  year={2013},
  publisher={Springer}
}

@article{doi:10.1137/0710036,
author = {Golub, G. H. and Pereyra, V.},
title = {The Differentiation of Pseudo-Inverses and Nonlinear Least Squares Problems Whose Variables Separate},
journal = {SIAM Journal on Numerical Analysis},
volume = {10},
number = {2},
pages = {413-432},
year = {1973},
doi = {10.1137/0710036},

URL = { 
    
        https://doi.org/10.1137/0710036
    
    

},
eprint = { 
    
        https://doi.org/10.1137/0710036
    
    

}
}

@article{zhou2021class,
  title={A class of line search-type methods for nonsmooth convex regularized minimization},
  author={Zhou, Weijun},
  journal={Soft Computing},
  volume={25},
  number={10},
  pages={7131--7141},
  year={2021},
  publisher={Springer}
}

@ARTICLE{1284395,
  author={Zhou Wang and Bovik, A.C. and Sheikh, H.R. and Simoncelli, E.P.},
  journal={IEEE Transactions on Image Processing}, 
  title={Image quality assessment: from error visibility to structural similarity}, 
  year={2004},
  volume={13},
  number={4},
  pages={600-612},
  doi={10.1109/TIP.2003.819861}}

@book{doi:10.1137/1.9781611975604,
author = {Gill, Philip E. and Murray, Walter and Wright, Margaret H.},
title = {Practical Optimization},
publisher = {Society for Industrial and Applied Mathematics},
year = {2019},
doi = {10.1137/1.9781611975604},
address = {Philadelphia, PA},
edition   = {},
URL = {https://epubs.siam.org/doi/abs/10.1137/1.9781611975604},
eprint = {https://epubs.siam.org/doi/pdf/10.1137/1.9781611975604}
}

@article{RUDIN1992259,
title = {Nonlinear total variation based noise removal algorithms},
journal = {Physica D: Nonlinear Phenomena},
volume = {60},
number = {1},
pages = {259-268},
year = {1992},
issn = {0167-2789},
doi = {https://doi.org/10.1016/0167-2789(92)90242-F},
url = {https://www.sciencedirect.com/science/article/pii/016727899290242F},
author = {Leonid I. Rudin and Stanley Osher and Emad Fatemi}
}

@article{51791361-8fe2-38d5-959f-ae8d048b490d,
 ISSN = {00359246},
 URL = {http://www.jstor.org/stable/2346178},
 author = {Robert Tibshirani},
 journal = {Journal of the Royal Statistical Society. Series B (Methodological)},
 number = {1},
 pages = {267--288},
 publisher = {[Royal Statistical Society, Oxford University Press]},
 title = {Regression Shrinkage and Selection via the Lasso},
 urldate = {2024-12-05},
 volume = {58},
 year = {1996}
}

@Inbook{Liu2023,
author="Liu, Zhifang
and Duan, Yuping
and Wu, Chunlin
and Tai, Xue-Cheng",
editor="Chen, Ke
and Sch{\"o}nlieb, Carola-Bibiane
and Tai, Xue-Cheng
and Younes, Laurent",
title="On Variable Splitting and Augmented Lagrangian Method for Total Variation-Related Image Restoration Models",
bookTitle="Handbook of Mathematical Models and Algorithms in Computer Vision and Imaging: Mathematical Imaging and Vision",
year="2023",
publisher="Springer International Publishing",
address="Cham",
pages="503--549",
isbn="978-3-030-98661-2",
doi="10.1007/978-3-030-98661-2_84",
url="https://doi.org/10.1007/978-3-030-98661-2_84"
}

@article{PRupdate,
     author = {Polak, E. and Ribiere, G.},
     title = {Note sur la convergence de m\'ethodes de directions conjugu\'ees},
     journal = {Revue fran\c{c}aise d'informatique et de recherche op\'erationnelle. S\'erie rouge},
     pages = {35--43},
     publisher = {Dunod},
     address = {Paris},
     volume = {3},
     number = {R1},
     year = {1969},
     mrnumber = {255025},
     zbl = {0174.48001},
     language = {fr},
     url = {http://www.numdam.org/item/M2AN_1969__3_1_35_0/}
}

@Inbook{Caselles2015,
author="Caselles, V.
and Chambolle, A.
and Novaga, M.",
editor="Scherzer, Otmar",
title="Total Variation in Imaging",
bookTitle="Handbook of Mathematical Methods in Imaging",
year="2015",
publisher="Springer New York",
address="New York, NY",
pages="1455--1499",
isbn="978-1-4939-0790-8",
doi="10.1007/978-1-4939-0790-8_23",
url="https://doi.org/10.1007/978-1-4939-0790-8_23"
}

@article{HAUFE2008726,
title = {Combining sparsity and rotational invariance in EEG/MEG source reconstruction},
journal = {NeuroImage},
volume = {42},
number = {2},
pages = {726-738},
year = {2008},
issn = {1053-8119},
doi = {https://doi.org/10.1016/j.neuroimage.2008.04.246},
url = {https://www.sciencedirect.com/science/article/pii/S1053811908005144},
author = {Stefan Haufe and Vadim V. Nikulin and Andreas Ziehe and Klaus-Robert Müller and Guido Nolte}
}

@article{OU2009932,
title = {A distributed spatio-temporal EEG/MEG inverse solver},
journal = {NeuroImage},
volume = {44},
number = {3},
pages = {932-946},
year = {2009},
issn = {1053-8119},
doi = {https://doi.org/10.1016/j.neuroimage.2008.05.063},
url = {https://www.sciencedirect.com/science/article/pii/S1053811908007155},
author = {Wanmei Ou and Matti S. Hämäläinen and Polina Golland}
}

@article{UUTELA1999173,
title = {Visualization of Magnetoencephalographic Data Using Minimum Current Estimates},
journal = {NeuroImage},
volume = {10},
number = {2},
pages = {173-180},
year = {1999},
issn = {1053-8119},
doi = {https://doi.org/10.1006/nimg.1999.0454},
url = {https://www.sciencedirect.com/science/article/pii/S1053811999904548},
author = {K. Uutela and M. Hämäläinen and E. Somersalo}
}

@inproceedings{Diffellah2021ImageDA,
  title={Image denoising algorithms using norm minimization techniques},
  author={Nacira Diffellah and Tewfik Bekkouche and Rabah Hamdini},
  year={2021},
  url={https://api.semanticscholar.org/CorpusID:235827354}
}

@article{GABAY197617,
title = {A dual algorithm for the solution of nonlinear variational problems via finite element approximation},
journal = {Computers \& Mathematics with Applications},
volume = {2},
number = {1},
pages = {17-40},
year = {1976},
issn = {0898-1221},
doi = {https://doi.org/10.1016/0898-1221(76)90003-1},
url = {https://www.sciencedirect.com/science/article/pii/0898122176900031},
author = {Daniel Gabay and Bertrand Mercier}
}

@article{FRISTON20081104,
title = {Multiple sparse priors for the M/EEG inverse problem},
journal = {NeuroImage},
volume = {39},
number = {3},
pages = {1104-1120},
year = {2008},
issn = {1053-8119},
doi = {https://doi.org/10.1016/j.neuroimage.2007.09.048},
url = {https://www.sciencedirect.com/science/article/pii/S1053811907008786},
author = {Karl Friston and Lee Harrison and Jean Daunizeau and Stefan Kiebel and Christophe Phillips and Nelson Trujillo-Barreto and Richard Henson and Guillaume Flandin and Jérémie Mattout}
}

@article{8051087,
  author={Liu, Feng and Rosenberger, Jay and Lou, Yifei and Hosseini, Rahilsadat and Su, Jianzhong and Wang, Shouyi},
  journal={IEEE Transactions on Big Data}, 
  title={Graph Regularized EEG Source Imaging with In-Class Consistency and Out-Class Discrimination}, 
  year={2017},
  volume={3},
  number={4},
  pages={378-391},
  doi={10.1109/TBDATA.2017.2756664}}

@article{Sarvas1987BasicMA,
  title={Basic mathematical and electromagnetic concepts of the biomagnetic inverse problem.},
  author={Jukka Sarvas},
  journal={Physics in medicine and biology},
  year={1987},
  volume={32 1},
  pages={
          11-22
        },
  url={https://api.semanticscholar.org/CorpusID:40767332}
}

@book{rankdeficient1998,
  author    = {Per Christian Hansen},
  title     = {Rank-deficient and discrete ill-posed problems: numerical aspects of linear inversion},
  publisher = {SIAM},
  year      = {1998},
}

@article{pascual2002standardized,
  title={Standardized low-resolution brain electromagnetic tomography (sLORETA): technical details},
  author={Pascual-Marqui, Roberto Domingo and others},
  journal={Methods Find Exp Clin Pharmacol},
  volume={24},
  number={Suppl D},
  pages={5--12},
  year={2002}
}

@book{bardsley,
author = {Johnathan M. Bardsley},
title = {Computational Uncertainty Quantification for Inverse Problems},
publisher = {SIAM},
year = {2018}}

@article{HUANG2016150,
title = {The New York Head—A precise standardized volume conductor model for EEG source localization and tES targeting},
journal = {NeuroImage},
volume = {140},
pages = {150-162},
year = {2016},
note = {Transcranial electric stimulation (tES) and Neuroimaging},
issn = {1053-8119},
doi = {https://doi.org/10.1016/j.neuroimage.2015.12.019},
url = {https://www.sciencedirect.com/science/article/pii/S1053811915011325},
author = {Yu Huang and Lucas C. Parra and Stefan Haufe}
}

@article{NYHEAD,
author = {Haufe, Stefan and Huang, Yu and Parra, Lucas},
year = {2015},
month = {08},
pages = {5744-7},
title = {A highly detailed FEM volume conductor model based on the ICBM152 average head template for EEG source imaging and TCS targeting},
volume = {2015},
journal = {Conference proceedings: ... Annual International Conference of the IEEE Engineering in Medicine and Biology Society. IEEE Engineering in Medicine and Biology Society. Conference},
doi = {10.1109/EMBC.2015.7319697}
}

@article{Michel2019EEGSI,
  title={EEG Source Imaging: A Practical Review of the Analysis Steps},
  author={Christoph M. Michel and Denis Brunet},
  journal={Frontiers in Neurology},
  year={2019},
  volume={10},
  url={https://api.semanticscholar.org/CorpusID:93003798}
}

@article{NOACHTAR200922,
title = {The role of EEG in epilepsy: A critical review},
journal = {Epilepsy \& Behavior},
volume = {15},
number = {1},
pages = {22-33},
year = {2009},
note = {Management of Epilepsy: Hope and Hurdles},
issn = {1525-5050},
doi = {https://doi.org/10.1016/j.yebeh.2009.02.035},
url = {https://www.sciencedirect.com/science/article/pii/S1525505009000924},
author = {Soheyl Noachtar and Jan Rémi}
}

@article{gramfort2012mixed,
  author    = {Alexandre Gramfort and Matthieu Kowalski and Matti H{\"a}m{\"a}l{\"a}inen},
  title     = {Mixed-norm estimates for the M/EEG inverse problem using accelerated gradient methods},
  journal   = {Physics in Medicine and Biology},
  volume    = {57},
  number    = {7},
  pages     = {1937--1961},
  year      = {2012},
  doi       = {10.1088/0031-9155/57/7/1937},
  url       = {https://doi.org/10.1088/0031-9155/57/7/1937},
  month     = apr,
  publisher = {IOP Publishing},
  issn      = {1361-6560},
  pmid      = {22421459},
  pmc       = {PMC3566429},
  note      = {Epub 2012 Mar 16}
}

@article{Invrev,
author = {Grech, Roberta and Cassar, Tracey and Muscat, Joseph and Camilleri, Kenneth and Fabri, Simon and Zervakis, Michalis and Xanthopoulos, Petros and Sakkalis, Vangelis and Vanrumste, Bart},
year = {2008},
month = {12},
pages = {25},
title = {Review on solving the inverse problem in EEG source analysis},
volume = {5},
journal = {Journal of neuroengineering and rehabilitation},
doi = {10.1186/1743-0003-5-25}
}

@book{hadamard1923lectures,
  title={Lectures on Cauchy's problem in linear partial differential equations},
  author={Hadamard, Jacques},
  volume={15},
  year={1923},
  publisher={Yale university press}
}

@article{boyd2011distributed,
  title={Distributed optimization and statistical learning via the alternating direction method of multipliers},
  author={Boyd, Stephen and Parikh, Neal and Chu, Eric and Peleato, Borja and Eckstein, Jonathan and others},
  journal={Foundations and Trends{\textregistered} in Machine learning},
  volume={3},
  number={1},
  pages={1--122},
  year={2011},
  publisher={Now Publishers, Inc.}
}

@inproceedings{hore2010image,
  title={Image quality metrics: PSNR vs. SSIM},
  author={Hore, Alain and Ziou, Djemel},
  booktitle={2010 20th international conference on pattern recognition},
  pages={2366--2369},
  year={2010},
  organization={IEEE}
}

@ARTICLE{142641,
  author={Wang, J.-Z. and Williamson, S.J. and Kaufman, L.},
  journal={IEEE Transactions on Biomedical Engineering}, 
  title={Magnetic source images determined by a lead-field analysis: the unique minimum-norm least-squares estimation}, 
  year={1992},
  volume={39},
  number={7},
  pages={665-675},
  doi={10.1109/10.142641}}

@article{PASCUALMARQUI199449,
title = {Low resolution electromagnetic tomography: a new method for localizing electrical activity in the brain},
journal = {International Journal of Psychophysiology},
volume = {18},
number = {1},
pages = {49-65},
year = {1994},
issn = {0167-8760},
doi = {https://doi.org/10.1016/0167-8760(84)90014-X},
url = {https://www.sciencedirect.com/science/article/pii/016787608490014X},
author = {R.D. Pascual-Marqui and C.M. Michel and D. Lehmann}
}

@article{zou2005regularization,
  title={Regularization and variable selection via the elastic net},
  author={Zou, Hui and Hastie, Trevor},
  journal={Journal of the Royal Statistical Society Series B: Statistical Methodology},
  volume={67},
  number={2},
  pages={301--320},
  year={2005},
  publisher={Oxford University Press}
}

@book{hansen2010discrete,
  title={Discrete Inverse Problems: Insight and Algorithms},
  author={Hansen, Per Christian},
  year={2010},
  publisher={SIAM}
}

@article{LOAN200085,
title = {The ubiquitous Kronecker product},
journal = {Journal of Computational and Applied Mathematics},
volume = {123},
number = {1},
pages = {85-100},
year = {2000},
note = {Numerical Analysis 2000. Vol. III: Linear Algebra},
issn = {0377-0427},
doi = {https://doi.org/10.1016/S0377-0427(00)00393-9},
url = {https://www.sciencedirect.com/science/article/pii/S0377042700003939},
author = {Charles F.Van Loan}
}

@book{LFMBOOK,
author = {Malmivuo, Jaakko and Plonsey, Robert},
year = {1995},
month = {10},
pages = {},
title = {Bioelectromagnetism - Principles and Applications of Bioelectric and Biomagnetic Fields},
isbn = {978-0195058239},
journal = {Bioelectromagnetism: Principles and Applications of Bioelectric and Biomagnetic Fields},
doi = {10.1093/acprof:oso/9780195058239.001.0001}
}

@inproceedings{GLW1975,
  title={Sur l'approximation, par {\'e}l{\'e}ments finis d'ordre un, et la r{\'e}solution, par p{\'e}nalisation-dualit{\'e} d'une classe de probl{\`e}mes de Dirichlet non lin{\'e}aires},
  author={Roland Glowinski and A. Marroco},
  year={1975},
  url={https://api.semanticscholar.org/CorpusID:124034911}
}

@article{10.1145/355984.355989,
author = {Paige, Christopher C. and Saunders, Michael A.},
title = {LSQR: An Algorithm for Sparse Linear Equations and Sparse Least Squares},
year = {1982},
issue_date = {March 1982},
publisher = {Association for Computing Machinery},
address = {New York, NY, USA},
volume = {8},
number = {1},
issn = {0098-3500},
url = {https://doi.org/10.1145/355984.355989},
doi = {10.1145/355984.355989},
journal = {ACM Trans. Math. Softw.},
month = mar,
pages = {43–71},
numpages = {29}
}

@misc{manoel2018approximatemessagepassingconvexoptimization,
      title={Approximate message-passing for convex optimization with non-separable penalties}, 
      author={Andre Manoel and Florent Krzakala and Gaël Varoquaux and Bertrand Thirion and Lenka Zdeborová},
      year={2018},
      eprint={1809.06304},
      archivePrefix={arXiv},
      primaryClass={stat.ML},
      url={https://arxiv.org/abs/1809.06304}, 
}

@article{10.1093/comjnl/7.2.149,
    author = {Fletcher, R. and Reeves, C. M.},
    title = {Function minimization by conjugate gradients},
    journal = {The Computer Journal},
    volume = {7},
    number = {2},
    pages = {149-154},
    year = {1964},
    month = {01},
    issn = {0010-4620},
    doi = {10.1093/comjnl/7.2.149},
    url = {https://doi.org/10.1093/comjnl/7.2.149},
    eprint = {https://academic.oup.com/comjnl/article-pdf/7/2/149/959725/070149.pdf},
}

@inbook{cfc5dc07425343f08f3c8ee5ae8f7ddc,
title = "Numerical optimization",
author = "Jorge Nocedal and Wright, {Stephen J.}",
year = "2006",
language = "English (US)",
series = "Springer Series in Operations Research and Financial Engineering",
publisher = "Springer Nature",
pages = "1--664",
booktitle = "Springer Series in Operations Research and Financial Engineering",
}

@article{fastADMM,
author = {Buccini, Alessandro},
year = {2022},
month = {01},
pages = {},
title = {Fast Alternating Direction Multipliers Method by Generalized Krylov Subspaces},
volume = {90},
journal = {Journal of Scientific Computing},
doi = {10.1007/s10915-021-01727-1}
}

@article{doi:10.1137/080716542,
author = {Beck, Amir and Teboulle, Marc},
title = {A Fast Iterative Shrinkage-Thresholding Algorithm for Linear Inverse Problems},
journal = {SIAM Journal on Imaging Sciences},
volume = {2},
number = {1},
pages = {183-202},
year = {2009},
doi = {10.1137/080716542},

URL = { 
    
        https://doi.org/10.1137/080716542
    
    

},
eprint = { 
    
        https://doi.org/10.1137/080716542
    
    

}
}

@misc{bianchi2021graphlaplacianimagedeblurring,
      title={Graph Laplacian for image deblurring}, 
      author={Davide Bianchi and Alessandro Buccini and Marco Donatelli and Emma Randazzo},
      year={2021},
      eprint={2102.10327},
      archivePrefix={arXiv},
      primaryClass={math.NA},
      url={https://arxiv.org/abs/2102.10327}, 
}

\begin{appendix}

\section{Supplemental Information}

\subsection{Implementation Details}\label{appendix:algo_options}

Algorithms implemented for this study will be made available on Github upon acceptance.  \Cref{tab:algo_opts} summarizes choices for algorithm implementation unless otherwise mentioned.  Numerical experiments indicated that these choices resulted in the most stable performance of our methods. Note that we input two separate sets of parameters for ${ \tt VPAL}_{\tt W}$--one for the initial window ({\tt optionsInit}), and one for subsequent windows ({\tt optionsLoop}).

\begin{table}[htb!]
    \centering
    \begin{tabular}{c||c|c|c|c|c}
         & {\tt ADMM} & {\tt VPAL} & {\tt FISTA} & ${ \tt VPAL}_{\tt W}$ & ${\tt VPAL}_{\tt W}$ \\

         & & & & \small{{\tt (optionsInit)}} &\small{{\tt (optionsLoop)}} \\
         \hline
         \hline

         step size selection &   & linearized & 
        & linearized & linearized \\
         \hline 
         {\tt tol} & $10^{-5}$ & $10^{-5}$ & $10^{-5}$ & $10^{-5}$  & $10^{-5}$ \\
         \hline

         {\tt maxIter} & 1,000 & 1,000 & 1,000  & 1,000 & 100 \\ \hline

         $\lambda$ &$10^{-1}$ & $10^{-5}$ & $10^{-3}$ & $10^{-5}$ & $10^{-5}$ \\
         \hline
         $\mu$ & $10^{-3}$ & $10^{-3}$ & $10^{-4}$&  $10^{-3}$ & $10^{-3}$ \\
         \hline
         $\eta$ & 10 & 10 & 10 & 10 & 10  \\
         \hline
         $L$ & &  &1,000 & &  \\
         \hline
         
    \end{tabular}
    \caption{Default parameters used in the implementation of algorithms unless otherwise mentioned.  $\lambda$ and $\mu$ were selected according to the procedure outlined in \Cref{section:paramselection}.}
    \label{tab:algo_opts}
\end{table}

In our implementation of {\tt ADMM}, we formulate the least-squares problem as a Sylvester Equation system, and use MATLAB's specialized Sylvester Equation ({\tt sylv}) solver to update $\bfX$ according to the method described in \Cref{appendix:admm_update}. For all algorithms, we use standard stopping criteria--see \cite{doi:10.1137/1.9781611975604}.

\subsection{Hyperparameter Study}\label{section:paramselection}

Hyperparameter selection is challenging in inverse problems and several strategies exist for the estimation of the optimal regularization and penalty parameters $\mu, \lambda, \eta$ \cite{hansen2010discrete, rankdeficient1998, bardsley, chung2022variable}.  Most existing approaches are only applicable when there are two parameters--one for the scaling of the $\ell_1$ term and one for the scaling in the penalty term of the augmented Lagrangian.  The introduction of an additional $\ell_2$ time regularization parameter makes these approaches more challenging and here we omit  a direct estimation of those parameters.  Furthermore, the Lipschitz constant used in the {\tt FISTA} algorithm is difficult to estimate and thus can be considered another hyperparameter that requires tuning.  As parameter estimation is not the focus of this paper, we adopt a naive brute-force approach for the estimation of suitable parameters and admit the limitation that rigorous statistical analysis is required for more optimal computation of these values.

Parameters used in the comparison experiment were selected by varying $\mu$ and $\lambda$ over multiple values and selecting the $(\mu, \lambda)$ pair that produced the smallest relative error after 1,000 iterations.  Experiments indicated that $[10^{-5}, 10^2]$ is a suitable range for the selection of both values and that the value of $\eta$ had little effect on the results of each algorithm as long as it was large enough. We used $\eta = 10$ in our experiments.  \Cref{fig:heatmaps} shows reconstruction errors for different choices of $\mu$ and $\lambda$ over a range of values. Selection of the Lipschitz constant, $L$ in {\tt FISTA}, is crucial to prevent divergence of the algorithm.  Overestimation of $L$ was found to promote convergence, and so a relatively large value, $L = 1,\!000$, was used for numerical experiments.

\begin{figure}[hbt!]
    \centering
    \includegraphics[scale = 0.3]{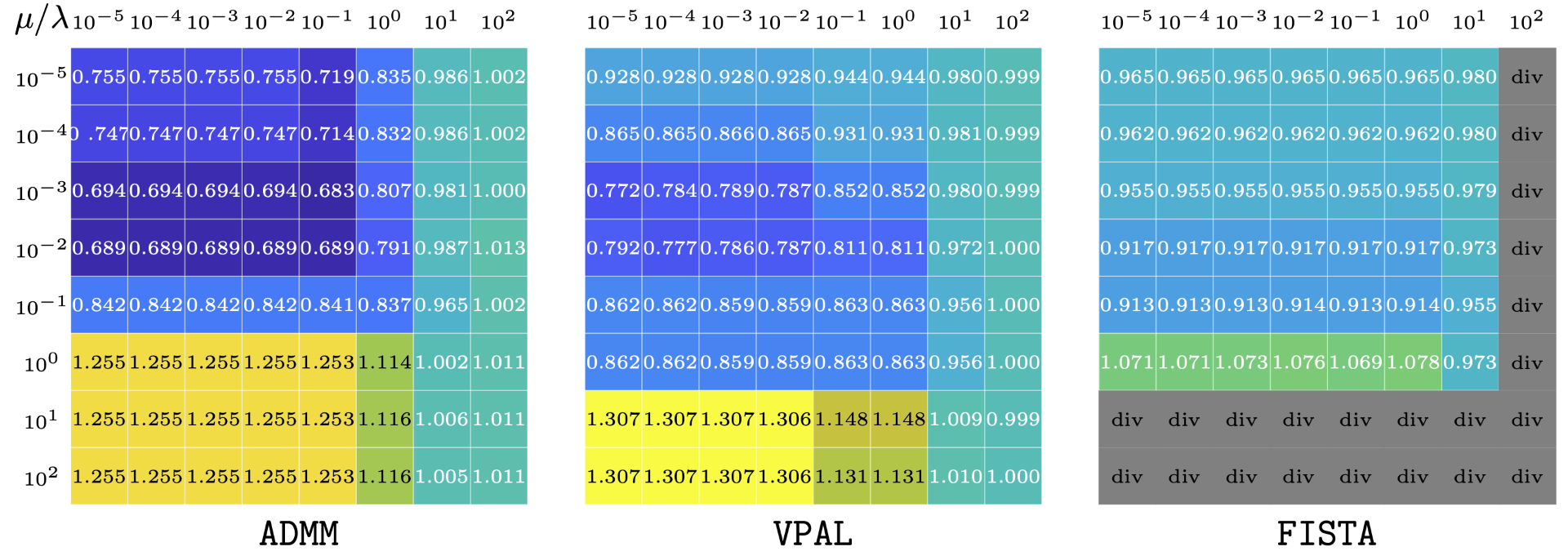}
    \caption{Heatmaps representing the relative error for the output of {\tt ADMM}, {\tt VPAL} and {\tt FISTA} as parameters $\mu$, $\lambda$ are varied.  {\tt div} refers to divergence in the algorithm for the respective parameters. The heatmap for {\tt VPAL$_{\tt W}$} experiences a similar behavior as {\tt VPAL} and is omitted.}
    \label{fig:heatmaps}
\end{figure}

The results presented in  \Cref{section:paramselection} indicate that {\tt ADMM} provides the best reconstruction results for small datasets and is most robust with respect to hyperparameter selection. {\tt FISTA} is susceptible to divergence for inappropriate selection of hyperparameters, and   has the additional disadvantage that $L$ is another parameter that must be tuned.  A clear limitation of our work is that we do not provide any analysis for the performance concerning $\eta$ or $L$ for {\tt FISTA}.

\section{Derivations}

\subsection{Equivalent Vectorized Forms of $f$ and \texorpdfstring{$\calL_{\rm aug}$}{Laug} }\label{appendix:vectorizedforms}

In this section, we show the derivation of the vectorized form of our objective function and augmented Lagrangian.  The objective function is given by \Cref{equation:objective},

\begin{align}
    f(\bfX) =  \thf \norm[\fro]{\bfL\bfX-\bfB}^2 + \tfrac{\lambda^2}{2}\norm[\fro]{\bfX\bfD_1}^2 + \mu\norm[1,1]{\bfD_2\bfX}.
\end{align}
If we define \begin{align}
\tilde{\bfL} = \bfI_T \otimes \bfL \in \bbR^{nT \times nT}, 
\end{align}
Kronecker identities imply \cite{LOAN200085} 
\begin{align}
\norm[\fro]{\bfL\bfX - \bfB}^2 =  \norm[\fro]{\tilde{\bfL}\bfx - \bfb}^2.
\end{align}
Analogously, for
\begin{align}
    \tilde{\bfD}_2 = \bfI_T \otimes \bfD_2 \in \bbR^{mT \times nT},
\end{align}
we have
\begin{align}
    \norm[1, 1]{\bfD_2\bfX}  = \norm[1]{\tilde{\bfD}_2\bfx}
\end{align}
Lastly, letting
\begin{align}
    \tilde{\bfD_1} = \bfD_1^\top \otimes \bfI_{n} \in \bbR^{n(T-1)\times nT},
\end{align} gives
\begin{align}
 \norm[\fro]{\bfX\bfD_1}^2 = \norm[2]{\tilde{\bfD_1}\bfx}^2.
\end{align}
All together, the vectorized form of our objective function can be written as

\begin{align}
    f(\bfx) = \frac{1}{2}\norm[2]{\tilde{\bfL}\bfx - \bfb} + \frac{\lambda^2}{2}\norm[2]{\tilde{\bfD_1}\bfx}^2 + \mu\norm[1]{\tilde{\bfD}_2\bfx}
\end{align}
\begin{align}
    = \frac{1}{2}\norm[2]{\begin{bmatrix}
        \tilde{\bfL} \\
        \tilde{\lambda\bfD_1}
    \end{bmatrix}\bfx - \begin{bmatrix}
        \bfb \\ \textbf{0}_{n(T-1)}
    \end{bmatrix}}^2 + \mu\norm[1]{\tilde{\bfD}_2\bfx}.
\end{align}
Denoting $\bfA = 
\begin{bmatrix}
\tilde{\bfL} \\
\tilde{\lambda\bfD_1}
\end{bmatrix} \in \bbR^{Tp + (T-1)n \times Tn}$ and $\bfw = \begin{bmatrix}
    \bfb \\ \textbf{0}_{n(T-1)}
\end{bmatrix} \in {\bbR^{Tp+n(T-1)}}$, we obtain 
\begin{align}
 f(\bfx) = \frac{1}{2}\norm[2]{\bfA\bfx - \bfw} + \mu\norm[1]{\tilde{\bfD}_2\bfx}.
\end{align}
Analogously, 
\begin{align}
    \calL_{\rm aug}(\bfx, \bfy ; \bfc) = \thf \norm[2]{\tilde{\bfL}\bfx - \bfb}^2 + \tfrac{\lambda^2}{2}\norm[2]{\tilde{\bfD}_1\bfx}^2 + \tfrac{\eta^2}{2}\norm[2]{\tilde{\bfD}_2\bfx - \bfy + \bfc}^2  + \mu\norm[1]{\bfy} - \norm[2]{\bfc}^2
\end{align}
\begin{align}
    = \thf\norm[2]{\bfA\bfx - \bfw}^2 + \tfrac{\eta^2}{2}\norm[2]{\tilde{\bfD}_2\bfx - \bfy + \bfc}^2  + \mu\norm[1]{\bfy} - \norm[2]{\bfc}^2.
\end{align}

\subsection{{\tt ADMM} Update}\label{appendix:admm_update}
Here we provide details on the formulation of the least-squares problem in {\tt ADMM} as the solution to a Sylvester equation for the EEG inverse problem. Let us restate the objective function in \Cref{equation:aug_lag}, i.e.,
\begin{align*}
    \calL_{\textnormal{aug}}(\bfX, \bfY ; \bfC) =  \thf \norm[\fro]{\bfL\bfX-\bfB}^2 + \tfrac{\lambda^2}{2}\norm[\fro]{\bfX\bfD_1}^2 + \mu\norm[1,1]{\bfY} + \tfrac{\eta^2}{2}\norm[\fro]{\bfD_2\bfX-\bfY+\bfC}^2 + \norm[\fro]{\bfC}^2.
\end{align*}
Consistent with the {\tt ADMM} approach, we want to minimize this function with respect to $\bfX$. Ignoring constant terms with respect to $\bfX$ and letting $\bfG = \bfY-\bfC$ we may simplify the corresponding optimization problem as 
\begin{align}
    \min_{\bfX} \ h(\bfX) = \min_{\bfX}\thf \norm[\fro]{\begin{bmatrix}
        \bfL \\ \eta \bfD_2
    \end{bmatrix}\bfX-\begin{bmatrix}
        \bfB \\ \eta \bfG
    \end{bmatrix}}^2 + \tfrac{\lambda^2}{2}\norm[\fro]{\bfX\bfD_1}^2 .
\end{align}
By defining
\begin{align}
    \bfH =\begin{bmatrix}
        \bfL \\ \eta \bfD_2
    \end{bmatrix} \textnormal{ and } \bfK = \begin{bmatrix}
        \bfB \\ \eta \bfG
    \end{bmatrix},
\end{align}
we can rewrite $h$ as
\begin{align}
h(\bfX) = \thf \norm[\fro]{\bfH\bfX - \bfK}^2 + \tfrac{\lambda^2}{2}\norm[\fro]{\bfX\bfD_1}^2.
\end{align}
The gradient with respect to $\bfX$ of the additive terms in $h$ can be computed as
\begin{align}
\frac{\partial}{\partial\bfX} \norm[\fro]{\bfX\bfD_1}^2 =  \frac{\partial}{\partial\bfX} \trace{\bfD_1^\top \bfX^\top \bfX \bfD_1} = \bfX \bfD_1\bfD_1^\top + \bfX\bfD_1\bfD_1^\top = 2 \bfX\bfD_1\bfD_1^\top
\end{align}
and
\begin{align}
    \frac{\partial}{\partial\bfX} \norm[\fro]{\bfH\bfX - \bfK}^2 &= \frac{\partial}{\partial\bfX}\trace{(\bfX^\top\bfH^\top - \bfK^\top)(\bfH\bfX - \bfK)}\\
    &= \frac{\partial}{\partial\bfX}( \trace{\bfX^\top\bfH^\top\bfH\bfX} - \trace{\bfX^\top\bfH^\top\bfK} - \trace{\bfK^\top\bfH\bfX} + \trace{\bfK^\top\bfK})\\
    &= 2\bfH^\top\bfH\bfX - \bfH^\top\bfK - \bfH^\top\bfK = 2\bfH^\top\bfH - 2\bfH^\top\bfK.
\end{align}
Hence the gradient of $h$ is given by 
\begin{align}
\frac{\partial h}{\partial \bfX} = \bfH^\top\bfH\bfX - \bfH^\top\bfK + \lambda^2 \bfX\bfD_1\bfD_1^\top.
\end{align}
According to first-order optimality condition and assumed convexity we obtain a global minimizer by equating to $\bf0$, i.e.,
\begin{align}\label{eq:sylv_solve}
 \bfH^\top\bfH\bfX  + \lambda^2 \bfX\bfD_1\bfD_1^\top =\bfH^\top\bfK.
\end{align}
Note this is the Sylvester Equation which can be solved efficiently using standard approaches. For our implementation of {\tt ADMM} we utilize MATLAB's specialized Sylvester Equation solver, {\tt sylv}.

\subsection{{\tt VPAL} Gradient Computation}\label{appendix:vpalncgupdate}
In this section, we provide details for the gradient computation used in $\Cref{alg:vpal}$.
We compute the gradient of \Cref{equation:aug_lag} with respect to $\bfX$ assuming $\bfC$ to be constant and $\bfY$ given according to soft-thresholding as computed in \Cref{eq:shrink}, i.e., $\bfY = \bfZ(\bfX)$.  We begin with the vectorized form of our augmented Lagrangian, \Cref{equation:vectorizedauglag}:

\begin{align*}
       \calL_{\textnormal{aug}}(\bfx, \bfy ; \bfc) =  \thf\norm[2]{\bfA\bfx - \bfw}^2 + \tfrac{\eta^2}{2}\norm[2]{\tilde{\bfD}_2\bfx - \bfy + \bfc}^2 + \mu \norm[1]{\bfy} - \tfrac{\eta^2}{2}\norm[2]{\bfc}^2.
\end{align*}

Taking the gradient with respect to $\bfx$, we obtain 
\begin{align}
\bfg = \tilde{\bfL}^\top(\tilde{\bfL}\bfx - \bfb) + \lambda^2 \tilde{\bfD}_1^\top \tilde{\bfD}_1\bfx  +  \eta^2\tilde{\bfD}_2^\top(\tilde{\bfD}_2\bfx - \bfy - \bfc),
\end{align}
where $\tilde{\bfL}, \tilde{\bfD_1}$ and $\tilde{\bfD}_2$ are defined the same way as in \Cref{appendix:vectorizedforms}. To avoid computation of the large matrices formed from the Kronecker product, we note
\begin{align}
    \tilde{\bfL}\bfx &= \vec{\bfL\bfX},
\\
\tilde{\bfD}_2\bfx &= \vec{\bfD_2\bfX},  
\\
    \tilde{\bfL}^\top\tilde{\bfL}\bfx 
    &= (\bfI_T\otimes \bfL^\top) \vec{\bfL\bfX} = \vec{\bfL^\top \bfL\bfX}
\\
\tilde{\bfD}_1\bfx &= (\bfD_1^\top \otimes \bfI_n)\bfx = \vec{\bfX\bfD_1} = \vec{\bfI\bfX\bfD_1}. \\
\tilde{\bfD}_1^\top \tilde{\bfD_1}\bfx = \tilde{\bfD}_1^\top \vec{\bfX\bfD_1} &= (\bfD_1 \otimes \bfI_n) \vec{\bfX\bfD_1} = \vec{\bfX\bfD_1\bfD_1^\top}.
\end{align}

All together, the vectorized gradient is computed as
\begin{align}
\bfg = \vec{\bfL^\top(\bfL\bfX - \bfB)} + \lambda^2 \vec{\bfX\bfD_1\bfD_1^\top} + \eta^2\vec{\bfD_2^\top(\bfD_2\bfX - \bfY + \bfC)}.
\end{align}

\end{appendix}

\end{document}